\documentclass[journal,twoside]{IEEEtran}
\usepackage{titlesec}
\usepackage{amssymb}
\usepackage{amsthm}
\usepackage{mathtools}
\usepackage{bm}
\usepackage{newunicodechar}
\newunicodechar{✓}{\ensuremath{\checkmark}}
\usepackage{booktabs}
\usepackage{array}
\newcolumntype{P}[1]{>{\raggedright\arraybackslash}p{#1}}
\usepackage{graphicx}
\usepackage{subcaption}
\usepackage{xcolor}
\usepackage{hyperref}
\usepackage{float}
\usepackage{placeins}
\usepackage{algorithm}
\usepackage{algpseudocode}
\titlespacing*{\section}{0pt}{3pt}{3pt}
\titlespacing*{\subsection}{0pt}{2pt}{2pt}
\titlespacing*{\subsubsection}{0pt}{1pt}{1pt}
\newtheorem{theorem}{Theorem}
\newtheorem{lemma}[theorem]{Lemma}
\newtheorem{proposition}[theorem]{Proposition}
\newtheorem{corollary}[theorem]{Corollary}
\newtheorem{definition}{Definition}
\newtheorem{remark}{Remark}
\newcommand{\R}{\mathbb{R}}
\newcommand{\N}{\mathbb{N}}
\newcommand{\cA}{\mathcal{A}}
\newcommand{\cM}{\mathcal{M}}
\newcommand{\cG}{\mathcal{G}}
\newcommand{\cS}{\mathcal{S}}
\newcommand{\cV}{\mathcal{V}}
\newcommand{\SCSI}{\mathcal{S}_{\mathrm{CSI}}}
\newcommand{\SCSIw}{\mathcal{S}_{\mathrm{CSI},w}}

\newcommand{\TG}{\mathcal{T}_{\mathcal{G}}}
\newcommand{\TS}{\mathcal{T}_{\mathcal{S}}}
\newcommand{\TV}{\mathcal{T}_{\mathcal{V}}}
\newcommand{\Xmat}{\mathbf{X}}
\newcommand{\Fmat}{\mathbf{F}}
\newcommand{\Cmat}{\mathbf{C}}
\newcommand{\Pmat}{\mathbf{P}}
\newcommand{\lmax}{\lambda_{\max}}
\newcommand{\Osys}{\Omega_{\mathrm{sys}}}
\newcommand{\Odia}{\Omega_{\mathrm{dia}}}
\newcommand{\norm}[1]{\left\|#1\right\|}
\newcommand{\abs}[1]{\left|#1\right|}
\newcommand{\CSI}{\mathrm{CSI}}

\begin{document}
\title{Attractor Domain Theory: Three Domains of Cardiovascular Nonlinear Dynamics}
\author{Timothy~Oladunni~and~Farouk~Ganiyu~Adewumi%
\thanks{A preliminary version of this work was posted as a preprint~\cite{oladunni2026adtpreprint}
(arXiv:2606.22039); the present manuscript supersedes that preliminary formulation.}%
\thanks{T.~Oladunni (corresponding author) is with the Department of Computer Science,
Morgan State University, Baltimore, MD~21251, USA
(e-mail: timothy.oladunni@morgan.edu).}%
\thanks{F.~Ganiyu Adewumi is with the Department of Computer Science,
Morgan State University, Baltimore, MD~21251, USA
(e-mail: faade1@morgan.edu).}% EMAIL: replace faade1@morgan.edu with F. G. Adewumi's correct address
\thanks{An earlier theory by the authors, Cardiac Stability Theory,
is available as arXiv:2604.23876. The Geometry-anchored empirical evaluation appears in Sections~VI--VIII of this work.}}
\markboth{Attractor Domain Theory}%
{Oladunni \& Adewumi: Attractor Domain Theory}
\maketitle
%─────────────────────────────────────────────────────────────────
\begin{abstract}
The cardiac attractor provides a phase-space representation of cardiovascular dynamics reconstructed from scalar observations under standard delay-embedding conditions. We introduce Attractor Domain Theory (ADT), a theory in which linked mathematical transformations expose distinct properties of reconstructed dynamics: the Geometry Domain is obtained through the delay-reconstruction transformation, while Ergodic and Variational transformations act on that geometry to characterize long-horizon organization and localized finite-time deformation. Thus, ADT specifies a three-representation operator architecture for reconstructed dynamics rather than retrospectively grouping existing nonlinear features; the architecture is proposed as an operational organization, not as a uniqueness or exhaustiveness theorem. A Conditional Domain Decomposition quantifies endpoint-specific incremental information with an explicit residual, and a domain-transformation impossibility theorem establishes that a property cannot be exactly recovered when an operator maps property-distinct inputs to the same representation. Empirical evaluation includes PPG analysis across 176{,}742 segments with 18 patient-level held-out BIDMC records (AUC 0.757 [0.686--0.828]) and an independent ECG analysis of 5{,}000 PTB-XL records (AUC 0.719 [0.686--0.750]). Cross-observation transfer is property-specific: the Lyapunov exponent shows subject-level ECG--PPG agreement ($\rho=0.703$, 95\% CI [0.503, 0.836]), whereas the predicted class-wide separation of invariant and observation-coupled descriptors is not supported ($p=0.54$). These results provide empirical support for attractor-domain analysis across cardiovascular measurements while leaving cross-observation recoverability as a property-specific empirical question.
\end{abstract}

\section{Introduction}
\label{sec:intro}

\subsection{Clinical Problem and Opportunity}

Cardiovascular disease (CVD) caused approximately 18.6
million deaths worldwide in 2019, more than any other major
disease category \cite{roth2020}. Smartphones can acquire
photoplethysmographic (PPG) signals using their optical sensors,
creating an accessible route for repeated cardiovascular measurement
\cite{allen2007,oladunni2026cst}. Yet many commonly used PPG-based
cardiac monitoring approaches emphasize peak-timing features such as
heart rate, RMSSD, and spectral heart rate variability (HRV) ratios
\cite{oladunni2026cst}. Waveform morphology, including the systolic
upstroke, dicrotic notch, and diastolic decay, is influenced by
cardiovascular and vascular properties and can contain information not
represented by peak timing alone \cite{allen2007,westerhof2009}.

These methods remain valuable for timing and variability analysis, but they do not determine which properties of reconstructed dynamics are retained, lost, or recoverable after a chosen mathematical transformation. ADT addresses this representation-level problem.

\subsection{Solution and Scientific Origin: Attractor Domain Theory}

We introduce Attractor Domain Theory (ADT) as an operator-defined organization of reconstructed nonlinear dynamics. For an observed scalar signal $x$, ADT defines the Geometry mathematical transformation $\TG$ as $\cG=\TG[x]$; the Ergodic and Variational mathematical transformations subsequently act on that reconstruction, defining $\cS=\TS[\cG]$ and $\cV=\TV[\cG,\Delta n]$, respectively. The central object of study is therefore not a new estimator for an existing nonlinear feature, but an explicitly specified transformation and its information consequences: which distinctions it retains, which it collapses, and what that implies for property recovery. The choice of three operators is motivated by three distinct analytical questions, state-space reconstruction, long-horizon organization, and local finite-time deformation, and is not claimed to be the only possible representation architecture.

This formulation leads to two testable consequences developed below. First, endpoint information can be decomposed conditionally across selected descriptors with an explicit residual for information not captured by them. Second, if a transformation identifies two inputs that differ in a property of interest, no downstream function of that transformed representation can recover that property exactly. The cardiovascular experiments provide an empirical system in which these consequences can be examined across PPG and ECG observations. Table~\ref{tab:gap} maps the open questions to the corresponding ADT results.

\subsection{Why Attractors? The Phase Space Representation}

Cardiovascular dynamics arise from interacting nonlinear physiological processes and have long been studied using nonlinear time-series methods \cite{goldberger2002}. For a dissipative cardiovascular model possessing a compact invariant attracting set $\cA \subset \R^m$, phase-space methods provide a natural setting for nonlinear dynamical analysis \cite{kantz2004}. ADT therefore states attractor existence as a modeling assumption rather than inferring it from physiological complexity alone.

\textbf{Why reconstructed dynamics?} Peak timing, spectral summaries, and other waveform-level measures remain useful, but they do not directly characterize the organization or deformation of a reconstructed state-space trajectory. Under the genericity, smoothness, embedding-dimension, and sampling conditions of delay-embedding theory \cite{takens1981,sauer1991,kantz2004}, a scalar time series can yield a reconstruction diffeomorphic to the underlying attractor. This reconstructed trajectory provides the common dynamical object on which ADT defines orbit geometry, long-horizon organization, and localized finite-time deformation. Corresponding descriptors such as correlation dimension, recurrence structure, and divergence rates then quantify distinct properties of that reconstructed dynamics.

Changes in cardiovascular dynamics can alter the morphology and temporal organization of an observed cardiovascular signal and, consequently, its reconstructed trajectory. ADT examines these changes at the representation level: the Geometry, Ergodic, and Variational mathematical transformations expose different properties of the reconstructed dynamics, from which endpoint-relevant descriptors can be evaluated. The physiological interpretation and cross-observation recoverability of any specific property remain empirical questions.

\subsection{Reconstruction and Cross-Observation}

Under suitable delay-embedding conditions, scalar cardiovascular signals can provide
an embedding of underlying dynamics. In ADT, ECG and PPG are modeled as distinct
observations of a common underlying cardiovascular dynamical system: ECG reflects
cardiac bioelectrical activity, whereas PPG reflects peripheral optical changes associated
with pulsatile blood volume \cite{malmivuo1995,allen2007}. The underlying cardiovascular dynamics are therefore observed through different measurement functions; the observation function can alter fidelity, scaling, noise sensitivity, and the recoverability of particular descriptors. Beat-to-beat morphological variations discarded by peak-timing methods can therefore carry state-space information not represented by timing alone.

\subsection{Existing Attractor Analysis: Unexplained Asymmetries}

The nonlinear cardiac dynamics literature has established that attractor
geometry carries clinical information. Lyapunov exponents measure trajectory
divergence rates \cite{rosenstein1993}; recurrence quantification
characterises recurrence structure and temporal organisation in reconstructed dynamics
\cite{zbilut1992,marwan2007recurrence}; correlation dimension characterises
attractor scaling \cite{grassberger1983characterization}, whereas Higuchi fractal
dimension measures fractal complexity of the time-series graph \cite{higuchi1988};
sample entropy quantifies time-series regularity and irregularity \cite{richman2000}.
Together, these provide established but distinct tools for nonlinear time-series analysis.

However, different attractor properties need not survive a change of observation function equally well. In paired ECG--PPG observations, for example, Lyapunov-exponent agreement can be substantial while recurrence-based agreement is weak (Remark~\ref{rem:cdh}). This raises a representation-level question not answered by applying individual nonlinear descriptors: what property-relevant distinctions are preserved or collapsed by a mathematical transformation of reconstructed dynamics?

\subsection{The Representation Theory Gap}

Delay embedding, correlation dimension, Lyapunov analysis, recurrence methods, ergodic theory, and finite-time dynamical analysis already provide rigorous tools for characterizing nonlinear systems \cite{takens1981,sauer1991,grassberger1983characterization,rosenstein1993,zbilut1992,pesin1977,haller2000,shadden2005}. The unresolved problem is not the absence of nonlinear descriptors, but the absence of a common representation-level treatment of how distinct mathematical transformations expose different dynamical properties and constrain downstream recoverability.

\subsection{Specific Gaps and Open Questions in Prior Work}

Cardiovascular studies have applied nonlinear measures to ECG, PPG, HRV, and
related physiological signals, demonstrating that entropy, recurrence, fractal,
and attractor-derived quantities can reveal structure beyond conventional timing
or spectral summaries \cite{goldberger2002,kantz2004,richman2000,pincus1994physiological,montano2009heart}.
However, these quantities are commonly introduced individually or assembled for
particular prediction tasks. Three representation-level questions therefore
remain open: \emph{which quantities describe reconstructed orbit geometry, which
characterize long-horizon dynamical organization, and which quantify localized
finite-time deformation; whether these representations contribute redundant or
complementary information for a specified cardiovascular endpoint; and which
properties remain recoverable when the same cardiovascular dynamics are observed
through different sensing functions.}

Our recent studies encounter these questions from complementary directions.
Cardiac Stability Theory (CST) \cite{oladunni2026cst} develops an axiomatic
formulation using selected attractor-derived quantities for cardiovascular
stability; Attractor--Vascular Coupling Theory (AVCT) \cite{oladunni2026avct}
examines relationships between attractor deformation and hemodynamic variation;
and Complementary Feature Domain (CFD) Theory \cite{oladunni2025cfd} establishes
that distinct signal representations can carry complementary rather than
redundant diagnostic information. These studies provide application-specific
motivation for a broader representation question, but ADT does not depend on
their individual empirical conclusions.

\subsection*{Six Contributions}

This paper makes six core contributions:

\begin{enumerate}

\item \textbf{Formal mathematical-transformation definitions}: We define three mathematical transformations ($\TG$, $\TS$, $\TV$) that map the cardiac signal and reconstructed
attractor into three operational dynamical representations: reconstructed orbit geometry, long-horizon ergodic organization, and localized finite-time deformation.

\item \textbf{Axiomatization of domain licensing}: We state the regularity and observability assumptions under which the three mathematical transformations are well defined, separating ideal delay-embedding conditions from finite-sample and noisy measurement limitations.

\item \textbf{Domain-specific mathematical interpretation}: We formalize how Geometry-domain descriptors of reconstructed orbit structure, Ergodic-domain descriptors of long-horizon organization, and Variational-domain descriptors of localized finite-time deformation answer different dynamical questions, and we state the assumptions required for subspace stability, cross-observation recoverability, and local hemodynamic calibration.

\item \textbf{Conditional Domain Decomposition}: We derive an exact chain-rule decomposition for selected domain descriptors and introduce the residual $\delta_y=I(X;y\mid G,S,V)$ as a direct test of endpoint-specific approximate sufficiency.

\item \textbf{Domain-transformation impossibility}: We prove that if an ADT mathematical transformation maps two distinct input dynamics to the same domain representation while a property differs between them, then that property cannot be exactly recovered from that domain representation alone. This gives a precise limit on what downstream descriptors can recover after a mathematical transformation has collapsed a distinction.

\item \textbf{Cross-modal applicability and property-specific transfer}:
We demonstrate that attractor-derived descriptors are useful in both PPG and ECG, while paired ECG--PPG analysis shows that cross-observation transfer is property-specific rather than universally invariant.

\end{enumerate}

\begin{table*}[!tp]
\centering
\caption{Gap Analysis: Prior Work, Open Gaps, ADT Contributions, and Impact}
\footnotesize
\label{tab:gap}
\renewcommand{\arraystretch}{1.1}
\setlength{\tabcolsep}{4pt}
\begin{tabular}{P{2.4cm} P{5.0cm} P{7.0cm}}
\toprule
\textbf{Prior Thread} & \textbf{Gap Left Open}
  & \textbf{ADT Contribution and Impact} \\
\midrule
Nonlinear features
\cite{rosenstein1993,zbilut1992,higuchi1988,richman2000}
& Existing nonlinear descriptors characterize specific dynamical properties, but do not provide a unified representation-level account of how descriptors derived from distinct mathematical transformations contribute incremental information for a specified endpoint.
& ADT defines linked Geometry, Ergodic, and Variational mathematical transformations; the Conditional Domain Decomposition (Thm.~\ref{thm:sufficiency}) quantifies endpoint-specific information from selected descriptors with an explicit residual. \\
HRV / peak-timing
\cite{task1996heart,montano2009heart}
& Peak-timing representations intentionally reduce the waveform to event timing and therefore do not retain the waveform samples used by delay-coordinate reconstruction.
& Geometry mathematical transformation: under valid delay-embedding conditions, $\Xmat$ can provide an embedding of the underlying attractor; peak timing does not retain the waveform samples used for state-space reconstruction (Def.~\ref{def:traj}). \\
FTLE / LCS
\cite{haller2000,shadden2005}
& FTLE/LCS methods are established in finite-time transport and continuum-flow analysis, but their role as a mathematical transformation of wearable-signal reconstructed cardiovascular dynamics and their relationship to BP remain open.
& Variational mathematical transformation ($\Osys$, $\Odia$); Prop.~\ref{thm:bp} motivates a local BP calibration hypothesis from phase-aggregated FTLE \cite{westerhof2009}. \\
Cross-modal transfer
& No criterion for which attractor invariants transfer
  ECG$\to$PPG vs.\ are corrupted by peripheral state
  $p(t)$; prior work targets translation, fusion, or
  empirical comparison.
& ADT provides a property-specific recoverability criterion (Cor.~\ref{cor:nonredundant}(iii)); paired ECG--PPG evidence shows strong transfer for $\lmax$ ($\rho=0.703$) but weak transfer for $R_{\mathrm{det}}$ ($\rho=0.123$), while class-wide separation was not supported (Remark~\ref{rem:cdh}) \cite{oladunni2026cst}. \\
CST, CFD, AVCT
\cite{oladunni2026cst,oladunni2025cfd,oladunni2026avct}
& Application-specific results motivate a broader representation-level
  organization connecting stability, hemodynamic deformation, and
  cross-observation behavior.
& ADT places these strands within a common mathematical-transformation representation theory while
  separating licensing conditions, local calibration assumptions,
  endpoint contribution, and recoverability claims. \\
\bottomrule
\end{tabular}
\end{table*}
\begin{table*}[!tp]
\centering
\caption{Research Questions, Theoretical Grounding, and Empirical Evidence}
\label{tab:rq}
\footnotesize
\renewcommand{\arraystretch}{1.1}
\setlength{\tabcolsep}{4pt}
\begin{tabular}{P{0.7cm} P{3.2cm} P{3.0cm} P{4.0cm} P{2.3cm}}
\toprule
\textbf{RQ} & \textbf{Research Question}
  & \textbf{Theoretical Grounding}
  & \textbf{Empirical Evidence}
  & \textbf{Outcome} \\
\midrule
RQ1
& How can endpoint-relevant information carried by ADT-derived descriptors be decomposed without assuming that the transformed representations are orthogonal or exhaustive?
& Thm.~\ref{thm:sufficiency} gives an exact chain-rule decomposition with explicit residual $\delta_y$; Prop.~\ref{prop:minmax} defines the operational scope of the three transformed representations.
& The formal decomposition establishes endpoint-specific information accounting. The empirical studies separately evaluate operational consequences of ADT-derived descriptors rather than treating ablation as an estimate of the conditional-information terms.
& \textbf{Exact information decomposition.} \\
\addlinespace
RQ2
& How does finite-window estimation affect interpretation of reconstruction-derived dynamical quantities?
& Prop.~\ref{prop:artifact} (Davis-Kahan rank separation);
  Condition~2 (smoothness and observability assumptions for $h$).
& Monotonically increasing finite-sample LLE estimates across scales ($0.086 \to 0.123$) motivate explicit separation of finite-window estimation from the asymptotic Lyapunov quantity.
& \textbf{Finite-window dependence observed.} \\
\addlinespace
RQ3
& Do the implemented ADT-derived descriptor composites contribute differently to endpoint discrimination?
& The geometry-anchored SCSI instantiation (Algorithm~\ref{alg:scsi}) begins with $\TG$ and combines descriptors derived from reconstructed dynamics; ablation tests predictive component importance rather than mutual information.
& $\Delta\mathrm{AUC} = -0.413$ on $C_{\mathrm{NL}}$ removal. Because $C_{\mathrm{NL}}$ mixes geometric and long-horizon/statistical descriptors, this result demonstrates nonuniform contribution among implemented composites without assigning the effect to a single ADT domain.
& \textbf{Differential descriptor contribution observed.} \\
\addlinespace
RQ4
& Do Bayesian-optimal embedding parameters match the prediction for the geometry-anchored SCSI reconstruction?
& Prediction~2: short $W^*$, high $m^*$ for
  within-beat $\cG$-dominant targets.
& $W^* = 128$ (shortest option), $m^* = 8$, converging
  at trial~82 of~300.
& \textbf{Consistent with Prediction 2.} \\
\addlinespace
RQ5
& Does the corrected evaluation protocol reveal genuine
  improvement over the unbiased baseline?
& Artifact analysis identifies three evaluation vulnerabilities: segment-level CV leakage, test-set normalisation leakage, and pooled AUC overweighting; their reported effects are evaluated separately rather than summed.
& Unbiased baseline AUC~$= 0.573$; optimised
  $\SCSI$ AUC~$= 0.757$ [0.686--0.828], NPV~$= 0.966$
  (18 held-out records; Table~\ref{tab:results}).
& \textbf{Improvement retained after corrected evaluation.} \\
\bottomrule
\end{tabular}
\end{table*}
Tables~\ref{tab:gap} and~\ref{tab:rq} together frame
the paper's contribution: Table~\ref{tab:gap} maps
five threads in the prior literature to their open
gaps and ADT resolutions; Table~\ref{tab:rq} pairs
each research question with its theoretical grounding
and the corresponding empirical evidence or test. The formal results establish ADT's mathematical structure, whereas the real-data experiments provide empirical support for its operational consequences: endpoint-relevant discrimination, nonuniform descriptor contribution, and property-specific cross-observation recoverability. The implemented composite ablations are therefore not interpreted as direct estimates of the conditional-information terms.

\section{Theoretical Lineage and Positioning}
\label{sec:lineage}
Table~\ref{tab:lineage} maps each prior theory to the core claim it established, the gap it left open, and the precise role ADT assigns it. ADT builds on delay reconstruction, ergodic analysis, finite-time dynamical analysis, and information-theoretic accounting while defining its own linked mathematically transformed representations of reconstructed cardiovascular dynamics.
\begin{table*}[!tp]
\centering
\caption{Theoretical Lineage: Prior Theories and ADT's Role.}
\footnotesize
\label{tab:lineage}
\renewcommand{\arraystretch}{1.1}
\setlength{\tabcolsep}{5pt}
\begin{tabular}{P{3.0cm} c P{3.8cm} P{3.8cm} P{4.0cm}}
\toprule
\textbf{Theory / Principle} & \textbf{Year}
  & \textbf{Core Claim}
  & \textbf{What It Leaves Open}
  & \textbf{ADT's Role} \\
\midrule
Takens--Sauer Delay-Embedding Results
\cite{takens1981,sauer1991}
& 1981/1991
& Under generic observation and embedding conditions,
  delay coordinates can provide an embedding
  of the underlying attractor.
& Which reconstructed invariants are
  useful for which tasks, and whether
  they carry redundant information.
& Provides the reconstruction basis for the Geometry mathematical transformation
  $\TG$ (Def.~\ref{def:traj}); under generic observation and embedding conditions,
  delay-coordinate theory supports the reconstructed geometry used as the common
  representation on which $\TS$ and $\TV$ subsequently act. \\
\addlinespace
Pesin's Entropy Formula
\cite{pesin1977}
& 1977
& Under Pesin's regularity conditions, the metric entropy
  $h_{KS}$ equals the sum of positive Lyapunov exponents.
& Distinguishing metric entropy from finite-partition
  stationary-distribution entropy.
& ADT therefore treats $H_{\mathrm{stat}}$ as an empirical
  partition-dependent statistic and does not identify it with
  $h_{KS}$; Prop.~\ref{thm:stability} uses it only as one
  selected descriptor. \\
\addlinespace
Recurrence Quantification
Analysis (RQA)
\cite{zbilut1992,marwan2007recurrence}
& 1992/2007
& Diagonal structure in the recurrence
  matrix quantifies deterministic
  phase-space organisation.
& Which RQA measures are non-redundant;
  why $R_{\mathrm{det}}$ transfers poorly
  ECG$\to$PPG.
& $R_{\mathrm{det}}$ is an Ergodic-domain recurrence descriptor;
  at a fixed threshold it is not generally diffeomorphism-invariant.
  Its poor transfer ($\rho = 0.123$ \cite{oladunni2026cst})
  motivates property-specific recoverability analysis. \\
\addlinespace
Finite-Time Lyapunov Exponents
\& Lagrangian Coherent
Structures
\cite{haller2000,shadden2005}
& 2000/2005
& FTLE fields and ridges are widely used diagnostics of finite-time stretching and can identify candidate coherent transport structures under additional LCS criteria.
& Formalization of localized finite-time deformation as a distinct transformed
  representation of reconstructed cardiovascular dynamics and evaluation of
  phase-conditioned deformation for the endpoints considered here.
& Defines the Variational mathematical transformation $\TV$ (Def.~\ref{def:TV}); $\Osys$ and $\Odia$ (Def.~\ref{def:phases}) are operational phase-conditioned regions over which FTLE is aggregated. They are not asserted to be LCS without separate coherent-structure verification. \\
\addlinespace
Shannon Information Theory
\cite{shannon1948}
& 1948
& Mutual information $I(X;Y)$ quantifies
  statistical dependence; the chain rule
  holds exactly.
& Application to endpoint-specific information accounting across selected attractor-domain descriptors.
& The chain rule proves
  Thm.~\ref{thm:sufficiency};
  \eqref{eq:sufficiency} is an identity,
  not an approximation. \\
\addlinespace
Cardiac Stability Theory (CST)
\cite{oladunni2026cst,oladunni2026scsi,oladunni2026avct}
& 2026
& Four axioms ground stability in
  attractor geometry; CSI combines
  $\lmax$, $R_{\mathrm{det}}$, $H$; BP
  maps affinely to expansion rates.
& A representation-level organization connecting stability,
  hemodynamic deformation, and cross-observation behavior.
& Provides prior cardiovascular motivation and empirical observations; ADT is defined
  independently and does not require CST as an axiomatic foundation. ADT separates
  formal identities, calibration assumptions, and recoverability hypotheses. \\
\bottomrule
\end{tabular}
\end{table*}

\section{Licensing Conditions for Attractor Domain Theory}
\label{sec:cst}
ADT is developed under four classes of licensing conditions: dynamical, observational, endpoint-interpretive, and cross-observation. Related assumptions were used previously in Cardiac Stability Theory (CST) \cite{oladunni2026cst}, but they are stated here independently because ADT does not depend on the empirical conclusions of CST.
\textbf{Condition~1 (Dynamical System).} The cardiovascular
system is governed by a dissipative, continuously
differentiable ($C^1$) nonlinear dynamical system
$f : \cM \to \cM$ admitting a compact absorbing ball
$B = \{z : \norm{z}^2 \le \rho + 1\}$.
Here $B$ is the physiological operating envelope: the
$\rho + 1$ bound ensures the boundary is strictly absorbing
(trajectories always point inward), so the cardiovascular
state always returns to $B$ regardless of starting
condition \cite{goldberger2002,kantz2004}. The
differentiability of $f$ is distinct from, and in addition
to, the smoothness of the observation function $h$ required
by Condition~2: $f$ governs how the true state evolves,
$h$ governs how that state is observed.
\emph{ADT role}: Dissipativity and boundedness of $f$
license the attractor model used in \eqref{eq:attractor}. A finite-dimensional
attractor satisfying an embedding-dimension condition is an additional
reconstruction requirement for $\TG$; it is not implied by dissipativity alone.
Proposition~\ref{prop:artifact} separately characterizes finite-sample subspace
stability through spectral separation. The Ergodic mathematical transformation $\TS$
(Section~\ref{sec:ergodic}) denotes a class of long-horizon or asymptotic
functionals and therefore does not require a single mandatory finite-state
construction; any chosen descriptor carries its own existence and estimation
conditions. Differentiability
of $f$ is the additional condition that licenses the
Variational mathematical transformation $\TV$ (Definition~\ref{def:TV}):
FTLE estimation presumes a locally linearizable flow, a
strictly stronger requirement than boundedness alone, since
the local deformation gradient $\hat{\Fmat}_n$
(eq.~\eqref{eq:Fhat}) is only a meaningful approximation to
$Df$ when $f$ is differentiable in a neighbourhood of
$\Xmat_n$. Systems that are dissipative but not everywhere
differentiable (e.g.\ admitting isolated non-smooth
transitions) may therefore support $\TG$ and $\TS$ while
$\TV$ requires separate justification; see
Section~\ref{sec:conclusion}.
\textbf{Condition~2 (Observable Projection).} The scalar
observable $x_n = h(z_n) + \varepsilon_n$, where
$h \in C^2(\cM, \R)$. For PPG, $h_{\mathrm{PPG}}$ additionally
depends on peripheral state $p_n$.
Here $x_n$ is the wearable signal, $z_n$ the true
cardiovascular state, $h$ the observation function
(for PPG: detecting the hemodynamic pressure wave via
optical absorption \cite{allen2007}), and $\varepsilon_n$
stochastic noise. The condition $h \in C^2(\cM, \R)$ supplies a smooth-observation assumption compatible with classical delay-embedding theory; smoothness alone does not guarantee a faithful reconstruction. Genericity of the observation, sufficient embedding dimension, properties of the underlying dynamics, and adequate sampling are also required \cite{takens1981,sauer1991}. \emph{ADT role}: these conditions license Definition~\ref{def:traj} in the ideal setting. Motion and optical artifacts can instead make the effective PPG observation process time-varying or contaminated and can violate the assumed smooth, fixed observation model. Reconstruction quality must therefore be assessed empirically before downstream $\TS$ or $\TV$ analyses are interpreted.
\textbf{Condition~3 (Endpoint-Interpretive Stability Hypothesis).}
For the cardiovascular-stability interpretation used in the present instantiation,
we hypothesize a bounded complexity band in selected long-horizon descriptors
$(\lmax, R_{\mathrm{det}}, H_{\mathrm{stat}})$ \cite{oladunni2026cst}, rather than a monotonic ``more complexity is healthier'' relation.
\emph{ADT role}: motivates selected Ergodic-Domain descriptors as one representation of long-horizon cardiovascular stability. The three quantities are selected scalar summaries of long-horizon dynamics
that are computable from finite orbits. Their invariance properties differ:
$\lmax$ can be conjugacy-invariant under the stated conditions, whereas
fixed-threshold $R_{\mathrm{det}}$ and finite-partition $H_{\mathrm{stat}}$
remain scale- or partition-dependent (Lemma~\ref{lem:invariance}).
\textbf{Condition~4 (Cross-Observation Recoverability Hypothesis).}
When ECG and PPG observe related cardiovascular dynamics through different
operators $h_{\mathrm{ECG}}$ and $h_{\mathrm{PPG}}$, ADT treats the extent of shared or distinct attractor-derived information as an empirical, property-specific recoverability question. This hypothesis is related to Complementary Feature Domain (CFD) theory \cite{oladunni2025cfd}, but ADT does not assume universal complementarity: the recoverability of any specific descriptor depends on its invariance properties, finite sampling, noise, scale, and observation-specific distortion.
\emph{ADT role}: Corollary~\ref{cor:nonredundant} motivates a
property-specific cross-observation recoverability hypothesis. Even when two
observations admit diffeomorphic ideal reconstructions, empirical agreement of a
particular descriptor remains contingent on finite sampling, noise, scale, and
observation-specific distortion. On BIDMC, $\lmax$ transfers strongly
($\rho = 0.703$), whereas fixed-threshold $R_{\mathrm{det}}$ transfers poorly
($\rho = 0.123$); the broader class-wide separation is not established
(Remark~\ref{rem:cdh}).

\section{Attractor Domain Theory}
\label{sec:adt}

Figure~\ref{fig:adt-partition} summarizes the mathematical-transformation structure that ADT formalizes. A scalar cardiovascular observation is first mapped by $\TG$ into the Geometry-domain reconstruction. The Ergodic and Variational transformations then map that reconstructed geometry into distinct long-horizon and finite-time representations. This transformation sequence precedes the choice of any particular nonlinear descriptor.

\begin{figure*}[!ht]
  \centering
  \includegraphics[width=0.96\linewidth]{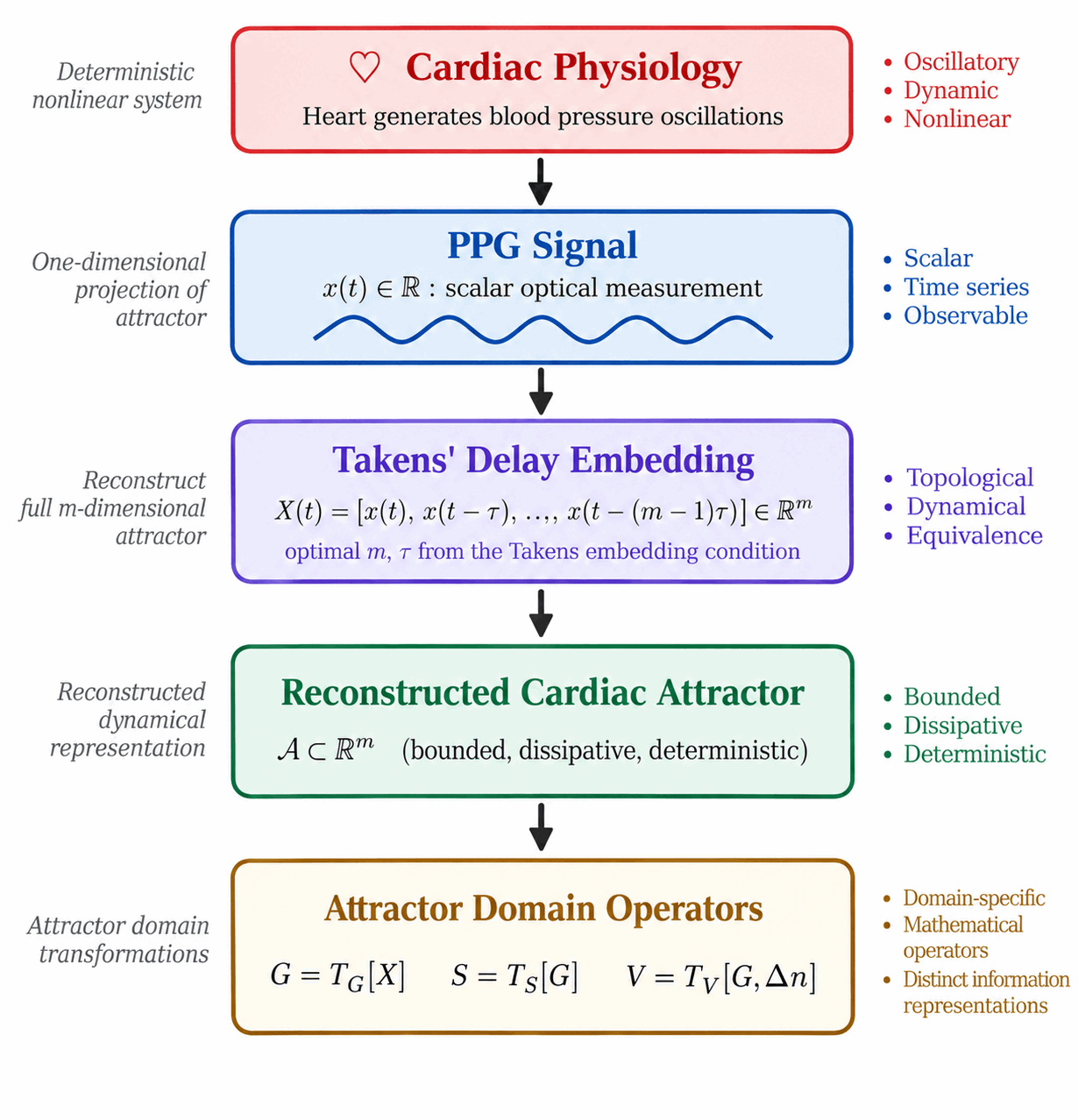}
  \vspace{8pt}
  \caption{Mathematical-transformation structure of ADT. The observed scalar signal $x$ is first mapped by the Geometry mathematical transformation $\TG$ to the reconstructed trajectory $X$, which provides the common state-space representation. The Ergodic mathematical transformation $\TS$ subsequently acts on this reconstructed geometry to expose long-horizon or asymptotic organization, while the Variational mathematical transformation $\TV$ acts on the same reconstructed geometry over a finite horizon $\Delta n$ to expose localized deformation. Domain-specific descriptors are functionals of these mathematically transformed representations; they do not define the domains by retrospective feature grouping. Endpoint-specific descriptor contributions are subsequently quantified by the Conditional Domain Decomposition. In the figure, $G$, $S$, and $V$ label the mathematically transformed domain representations; the finite descriptor vectors used in the information decomposition are measurements within those domains, not definitions of them.}
  \label{fig:adt-partition}
\end{figure*}

\subsection{Discrete Signal Model and Attractor}\label{sec:prelim}
Let $x \in \R^T$ denote the discrete cardiovascular time
series sampled at rate $f_s$, generated from a latent state
sequence $\{z_n\}$ under Condition~2:
\begin{equation}
  x_n = h(z_n) + \varepsilon_n.
  \label{eq:obs}
\end{equation}
The state evolves under a deterministic map $f : \cM \to \cM$
on a compact forward-invariant trapping region $\cM \subset \R^d$ with $f(\cM)\subseteq\cM$. Under Condition~1, the maximal forward-invariant attracting set used by ADT is:
\begin{equation}
  \cA = \bigcap_{n=0}^{\infty} f^n(\cM),
  \label{eq:attractor}
\end{equation}
where $f^n(\cM)$ denotes the image of $\cM$ after $n$
applications of $f$, and $\bigcap$ retains only the
states reachable after arbitrarily many steps, the
set the system never escapes regardless of how long it
runs. Under the additional compactness, forward-invariance, and attraction
assumptions adopted here, this construction represents the invariant attracting set
used by ADT. Neither proper inclusion $\cA\subsetneq\cM$ nor a particular
fractal dimension follows from dissipativity alone; the embedding dimension is
chosen separately to satisfy the applicable delay-reconstruction conditions.

\subsection{Formal Domain Definitions}
\label{sec:domains}

The three ADT domains are formally defined through linked mathematical transformations within the proposed architecture. The Geometry mathematical transformation acts on the observed scalar signal to construct the reconstructed trajectory. The Ergodic and Variational mathematical transformations subsequently act on that reconstructed geometry, with different temporal limits and codomains, to expose long-horizon organization and localized finite-time deformation. Specific nonlinear descriptors are functionals measured within these representations rather than the basis on which the domains are defined.

\subsubsection{The Geometry Domain \texorpdfstring{$\cG$}{G}}
\label{sec:geom}
\begin{definition}[Trajectory Matrix / $\TG$]
\label{def:traj}
Given embedding dimension $d \in \N$ and delay
$\tau \in \N$, the \emph{Geometry mathematical transformation}
$\TG : \R^T \to \R^{N \times d}$ constructs the trajectory
matrix:
\begin{equation}
  \Xmat = \TG[x], \quad
  \Xmat_{n,\cdot} = [x_n,\, x_{n+\tau},\, \ldots,\,
    x_{n+(d-1)\tau}],
  \label{eq:trajmat}
\end{equation}
where $N = T - (d-1)\tau$ and each row $\Xmat_n \in \R^d$
is a reconstructed state: the PPG sample at time $n$
placed alongside $d-1$ time-delayed copies of itself,
so that each row represents one estimated position of
the cardiovascular system in the reconstructed phase
space. In the noiseless setting, under the genericity, smoothness, sampling, and dimension conditions of delay-embedding theory \cite{takens1981,sauer1991}, the sequence $\{\Xmat_n\}$ can provide a diffeomorphic reconstruction of the underlying attractor for suitable $(d,\tau)$. For a smooth compact manifold of dimension $q$, Whitney's classical result gives an embedding bound of $2q+1$; for fractal attractors, Sauer--Yorke--Casdagli provide generic embedding conditions in terms of box-counting dimension \cite{sauer1991}. These are distinct settings and are not identified with one another here.
A diffeomorphism is a smooth bijection with a smooth inverse and preserves conjugacy-invariant dynamical quantities under the relevant regularity conditions. Fixed-scale recurrence statistics and partition-dependent entropies need not be invariant under arbitrary reparameterization. With finite sampling and acquisition noise, the observed trajectory matrix is only an approximation to the ideal reconstruction; reconstruction quality must therefore be evaluated empirically rather than assumed.
\end{definition}
\paragraph{Dimension Specification for PPG.}
The embedding dimension must be chosen relative to the intrinsic dimension of the attractor accessible through the observation. Electrical, autonomic, and vascular processes jointly shape cardiovascular observations, motivating a multidimensional physiological state; however, counting such processes does not mathematically determine $\dim(\cA)$ \cite{malmivuo1995,task1996heart,allen2007,westerhof2009}. In the present PPG data, empirical dimension estimates suggest an effective observable dimension near $3.5$, which is compatible with choosing $m^*=8$ as a conservative reconstruction heuristic. This is not a direct application of Whitney's manifold theorem and does not by itself establish an embedding guarantee.
\paragraph{Codomain Structure.}
The codomain of $\TG$ supports representative geometric objects and descriptors of the reconstructed attractor. The \emph{correlation dimension}
$D_2$, estimated from the slope of the log correlation sum
$\ln C(r)$ versus $\ln r$ \cite{kantz2004}, quantifies the
fractal complexity of $\Xmat$. The \emph{attractor boundary}
$\partial\cA$ is the index set of states where the empirical
density $\hat{\rho}(\Xmat_n)$ undergoes a sharp transition:
\begin{equation}
  \partial\cA = \{n : \norm{\nabla \hat{\rho}(\Xmat_n)}_2
    \ge \gamma_{\mathrm{thresh}}\},
  \label{eq:boundary}
\end{equation}
computed via finite differences on the $k$-nearest-neighbor
graph. In plain terms, $\partial\cA$ identifies the time
steps where the trajectory is near the edge of the
attractor rather than its core: the samples where
the PPG waveform is transitioning between dense and
sparse regions of the reconstructed phase space. The empirical density support provides one way to characterize where reconstructed states concentrate. The homeostasis proxy $\Omega_w$ introduced in Section~\ref{sec:features} is an implementation-specific quality/operating-envelope gate; it does not mathematically enforce the attractor boundary condition in Eq.~\eqref{eq:boundary}.
\paragraph{Quasi-Inverse.}
The reconstruction map $\TG^{\dagger} : \R^{N \times d} \to \R^T$
uses diagonal averaging
\cite{golyandina2001}:
\begin{equation}
  (\TG^{\dagger}[\Xmat])_n =
    \frac{1}{\abs{\mathcal{D}_n}}
    \sum_{(i,j) \in \mathcal{D}_n} \Xmat_{i,j},
  \label{eq:quasiinv}
\end{equation}
where $\mathcal{D}_n = \{(i,j) : i + (j-1)\tau = n\}$.
This is the diagonal averaging step of Singular Spectrum
Analysis \cite{golyandina2001}.
\paragraph{Representative Capability: Reconstruction and Subspace Stability.}
\begin{remark}[Reconstruction vs.\ Bijective Transform]
\label{rem:reconstruction}
$\TG$ is a \emph{reconstruction} operator, not a bijective transform. It has no exact inverse in the general formulation: distinct signal/reconstruction configurations can map to the same retained representation, and the domain outputs are not assumed to form a Hilbert-space decomposition. ADT therefore does not require bijectivity, isometry, or orthogonality among its domain transformations. What ADT requires is not that each domain be a bijection, nor that every domain contribute nonredundant information for every endpoint. Theorem~\ref{thm:sufficiency} instead provides an exact accounting in which domain-specific conditional increments may be zero or positive for a specified target $y$, with an explicit residual for information not captured by the selected descriptors. This is an informational decomposition, not an isometric or orthogonal Parseval identity. Diagonal averaging maps a Hankel-consistent trajectory matrix back to a signal. Applied directly to $\TG[x]$, it recovers the embedded samples up to boundary/coverage conventions; denoising requires an additional low-rank or structured projection before diagonal averaging. Accordingly, $\TG^{\dagger}$ should not be interpreted as removing noise by itself.
\end{remark}
\begin{proposition}[Conditional Geometry-Domain Subspace Stability]
\label{prop:artifact}
Let $x_{\mathrm{obs}}=x+\eta$ and $Y=\TG[x_{\mathrm{obs}}]=X+E$, where $X=\TG[x]$ has an $r$-dimensional signal subspace and $E=\TG[\eta]$ is an additive perturbation. If the clean signal subspace is separated by a nonzero singular-value gap
\begin{equation}
  \gamma = \sigma_r(X)-\sigma_{r+1}(X) > 0,
\end{equation}
and $\norm{E}_2 < \gamma$, then the principal-angle error between the clean and perturbed rank-$r$ subspaces is bounded, up to the standard Davis--Kahan/Wedin perturbation constants, by
\begin{equation}
  \norm{\sin\Theta(\widehat{\cM},\cM)}_2
  \lesssim \frac{\norm{E}_2}{\gamma-\norm{E}_2}.
  \label{eq:clean}
\end{equation}
Consequently, projection onto the estimated leading rank-$r$ subspace is stable only when the perturbation is small relative to the spectral gap. No orthogonality between physiological signal and artifact is assumed.
\end{proposition}
\begin{proof}
This is a direct application of standard invariant-subspace perturbation theory \cite{davis1970}. The required quantity is separation of the signal singular subspace from neighboring singular directions, not the dimensionality of the artifact. High-dimensional noise may overlap the signal subspace; the bound therefore conditions recovery on a measurable spectral gap and perturbation magnitude. \qed
\end{proof}
\begin{remark}[Geometry-Domain Quality Control]
Proposition~\ref{prop:artifact} motivates the reconstruction-quality controls used in the PPG extension (Proposition~\ref{prop:ppg}). The signal-quality gate $Q_w$ and operating-envelope proxy $\Omega_w$ are practical controls motivated by the need for adequate reconstruction stability; they do not directly estimate the spectral gap in Proposition~\ref{prop:artifact}. In the present SCSI implementation, failure of either gate triggers rejection because downstream estimates are treated as unreliable, not because ADT proves that no reconstruction could exist.
\end{remark}

\subsubsection{The Ergodic Domain \texorpdfstring{$\cS$}{S}}
\label{sec:ergodic}
\begin{definition}[Ergodic Mathematical Transformation / $\TS$]
\label{def:TS}
The \emph{Ergodic mathematical transformation} denotes the class of mappings
\[
  \TS : \R^{N \times d} \to \mathcal{Y}_{\cS}
\]
that map the reconstructed trajectory $\Xmat$ to a selected long-horizon or asymptotic representation in a chosen codomain $\mathcal{Y}_{\cS}$. Membership in $\cS$ is determined by the dynamical question, long-run organization, invariant-measure structure, recurrence structure, or asymptotic divergence, rather than by a single mandatory finite-state construction. Scalar or vector descriptors are then functionals of the selected Ergodic-domain representation; they do not define the domain itself.

A useful finite-state realization partitions $\R^d$ into $K$ bins $\{B_1, \ldots, B_K\}$ and constructs the empirical transition matrix
\begin{equation}
  P_{ij} = \frac{\#\{n : \Xmat_n \in B_i,\;
    \Xmat_{n+1} \in B_j\}}
    {\#\{n : \Xmat_n \in B_i\}}.
  \label{eq:transmat}
\end{equation}
When this matrix is primitive, its stationary distribution $\mu \in \R^K$ satisfies
$\mu^\top \Pmat = \mu^\top$, $\mu \ge 0$, and
$\norm{\mu}_1 = 1$, with uniqueness following from the Perron--Frobenius theorem. The pair $(\Pmat,\mu)$ is therefore one operational realization of $\TS$, not the complete codomain of the Ergodic Domain.
\end{definition}
\paragraph{Representative Ergodic-Domain Descriptors.}
Representative Ergodic-domain descriptors include long-horizon quantities computed from the reconstructed orbit and, when appropriate, from finite-state constructions such as $(\Pmat,\mu)$. In particular:
\emph{(i) Largest Lyapunov Exponent.} The mean divergence
rate of nearby orbits \cite{rosenstein1993}:
\begin{equation}
  \lmax = \frac{1}{\Delta n}
    \left\langle \ln \frac{\norm{\Xmat_{n+\Delta n} -
      \Xmat_{m+\Delta n}}}{\norm{\Xmat_n -
      \Xmat_m}} \right\rangle_{\mathcal{N}},
  \label{eq:lle}
\end{equation}
where $\mathcal{N}$ is the set of nearest-neighbor pairs
with temporal separation $> \lfloor\bar{T}_{\mathrm{beat}}/2
\rfloor$ (the stabilized estimator of Section~\ref{sec:features},
eq.~\eqref{eq:lambda}, correcting the Rosenstein boundary error).
\emph{(ii) Recurrence Determinism.} The fraction of
recurrence points forming diagonal structures
\cite{zbilut1992}:
\begin{equation}
  R_{\mathrm{det}} = \frac{\sum_{\ell \ge \ell_{\min}}
    \ell\,P(\ell)}{\sum_{\ell \ge 1}\ell\,P(\ell)},
  \label{eq:rdet}
\end{equation}
where $P(\ell)$ is the diagonal line length histogram of
$R_{ij} = \mathbf{1}[\norm{\Xmat_i - \Xmat_j} \le r]$.
$R_{\mathrm{det}}$ is a functional of the recurrence
structure of $\Xmat$: the diagonal line distribution
reflects how long the orbit persists near a given
region of phase space, a long-horizon recurrence property that depends on recurrence geometry and temporal sequencing of the reconstructed orbit. It is not determined by the stationary distribution $\mu$ or $H_{\mathrm{stat}}$ alone because recurrence determinism also depends on the temporal arrangement of recurrent states \cite{marwan2007recurrence,kantz2004}. This places
$R_{\mathrm{det}}$ within the Ergodic Domain rather
than constituting a fourth independent domain.
\emph{(iii) Stationary-Distribution Entropy.}
\begin{equation}
  H_{\mathrm{stat}} = -\sum_{i=1}^K \mu_i \ln \mu_i,
  \label{eq:hmu}
\end{equation}
is the Shannon entropy of the stationary distribution induced by the chosen finite partition. It is partition-dependent and must not be identified with the Kolmogorov--Sinai (metric) entropy $h_{KS}$. Under the regularity hypotheses of Pesin's entropy formula, $h_{KS}$--not $H_{\mathrm{stat}}$--is related to the sum of positive Lyapunov exponents \cite{pesin1977}. We therefore use $H_{\mathrm{stat}}$ only as an empirical finite-partition statistic.
\paragraph{Representative Capability: Long-Horizon Stability
Estimation.}
\begin{lemma}[Invariance Classification of Ergodic-Domain Functionals]
\label{lem:invariance}
Let $\Phi_1,\Phi_2$ be delay-coordinate reconstructions of the same
attractor $\cA$, related by the diffeomorphism
$\varphi = \Phi_2\circ\Phi_1^{-1}$. Then \emph{(i)} $\lmax$ is
invariant; \emph{(ii)} the correlation dimension $D_2$ and the
Kolmogorov--Sinai entropy $h_{KS}$ are invariant; \emph{(iii)}
$R_{\mathrm{det}}$ at a \emph{fixed} threshold $r$, and $H_{\mathrm{stat}}$ on
a \emph{fixed} partition $\{B_1,\dots,B_K\}$, are \textbf{not}
invariant: they are metric- and partition-dependent functionals of
$(\Pmat,\mu)$.
\end{lemma}
\begin{proof}
\emph{(i)} $\varphi$ induces a smooth conjugacy; on the compact set
$\Phi_1(\cA)$, $D\varphi$ is bounded with bounded inverse, so the
$\log\norm{D\varphi}$ terms are $O(1)$ and vanish in the limit
$\tfrac1n\log\norm{Df^n}$: Lyapunov exponents are conjugacy
invariants.
\emph{(ii)} On a compact set a $C^1$ diffeomorphism is
bi-Lipschitz, so $C_\varphi(cr) \le C(r) \le C_\varphi(Cr)$ and the
exponent $D_2 = \lim_{r\to0}\log C(r)/\log r$ is unchanged;
with the corresponding invariant measure pushed forward by $\varphi$, $h_{KS}$ is preserved under the induced measure-theoretic conjugacy.
\emph{(iii)} Bi-Lipschitz bounds preserve \emph{limits}, not
\emph{fixed scales}. Let $\varphi$ dilate a subregion by $c<1$.
Then $R_{ij} = \mathbf{1}[\norm{\Xmat_i-\Xmat_j}\le r]$ flips for
every pair with $\norm{\Xmat_i-\Xmat_j}\in(r,\,r/c]$, changing the
diagonal-line histogram $P(\ell)$ and hence $R_{\mathrm{det}}$.
Likewise $\varphi$ does not map $\{B_i\}$ to itself, so $\Pmat$,
$\mu$, $H_{\mathrm{stat}}$ change; a fixed-partition stationary-distribution entropy is not the metric entropy $h_{KS}$; refinement of the partition alone does not justify identifying the two. Thus invariance fails at fixed $(r,K)$ and must be established separately for the quantity of interest. \qed
\end{proof}

\begin{proposition}[CSI as an Ergodic-Domain Composite]
\label{thm:stability}
Let
\begin{equation}
  \mathbf{s} = [\lmax,\; R_{\mathrm{det}},\; H_{\mathrm{stat}}]^\top,
  \label{eq:svec}
\end{equation}
collect selected long-run dynamical and finite-partition descriptors. For normalized components and $\norm{\mathbf{w}}_1=1$,
\begin{equation}
  \CSI = \sigma(\mathbf{w}^\top \mathbf{s}) \in [0,1],
  \label{eq:csi}
\end{equation}
defines a bounded composite whenever $\sigma:\R\to[0,1]$ is bounded and monotone. The construction is one admissible summary of the selected descriptors; neither uniqueness nor completeness over all functionals of $(\Pmat,\mu)$ is asserted.
\end{proposition}
\begin{proof}
Boundedness follows from normalization and the range of $\sigma$. Measurability follows because the selected descriptors are computed from the reconstructed trajectory and its induced recurrence/transition statistics. By Lemma~\ref{lem:invariance}, the Lyapunov component can be conjugacy-invariant under the stated conditions, whereas fixed-threshold recurrence statistics and $H_{\mathrm{stat}}$ remain scale- and partition-dependent. Thus CSI is an operational composite, not a unique invariant of the Ergodic Domain. \qed
\end{proof}

\begin{corollary}[Cross-Observation Transfer Criterion]
\label{cor:transfer}
Let $h_1,h_2$ be $C^2$ observation functions on $\cA$ and $\psi$ a
scalar functional of the reconstruction. If $\psi$ is a
diffeomorphism invariant (Lemma~\ref{lem:invariance}(i)--(ii)),
then $\psi$ computed from $h_1$ and from $h_2$ agree in the
infinite-data, noise-free limit under the same temporal parametrization/sampling convention; if $\psi$ is fixed-scale or
fixed-partition (Lemma~\ref{lem:invariance}(iii)), no agreement is
implied. (Both reconstructions are diffeomorphic to $\cA$, hence to
each other; apply Lemma~\ref{lem:invariance}.)
\end{corollary}
\begin{remark}[Ergodic Weights and $\SCSI$ Ablation]
Proposition~\ref{thm:stability} contextualizes two empirical results
in Section~\ref{sec:validation}. The $C_{\mathrm{NL}}$ dominance
($\Delta$AUC~$= -0.413$ on removal) reflects that $C_{\mathrm{NL}}$ is a
geometry-anchored mixed Geometry/Ergodic complexity composite, not a pure Ergodic Domain
invariant of $(\Pmat,\mu)$; Proposition~\ref{thm:stability} does not claim completeness and therefore does not bound it, and
removing it removes the dominant predictive term for
tachypnea. The ECG-validated weights
$w_1 = 0.40$, $w_2 = 0.35$, $w_3 = 0.25$
\cite{oladunni2026cst} are the empirical solution to the
maximum-separability problem under the endpoint-specific stability hypothesis used for ECG,
not arbitrary choices.
\end{remark}
\begin{remark}[Cross-Observation Recoverability and CFD]
\label{rem:cdh}
ADT relates the Complementary Domain Hypothesis to the observation-function perspective of CFD \cite{oladunni2025cfd} while treating cross-observation recoverability as an empirical property-specific question. Since $h_{\mathrm{PPG}}$ carries $p_n$, it fixes different
effective scales than $h_{\mathrm{ECG}}$; by
Corollary~\ref{cor:transfer} the invariants agree across modalities
while fixed-scale functionals need not, replacing the informal
``sensitivity to $p_n$'' argument. On the $53$ paired BIDMC
records ($4{,}826$ of $4{,}833$ paired windows, after excluding
those without a valid heart-rate estimate), recomputed per subject with
heart-rate-partialled Spearman correlations, the invariants
transfer ($\lmax$: $\rho = 0.703$, $95\%$ CI $[0.503, 0.836]$;
$D_2$: $\rho = 0.395$) while the fixed-scale functionals do not
($R_{\mathrm{det}}$: $\rho = 0.123$; amplitude $H$:
$\rho = -0.076$; both CIs span zero), at subject-level
reliability $\ge 0.979$ for every descriptor. Two caveats: permutation and sample
entropy are invariants that also fail to transfer ($\rho = 0.130$,
$0.118$) while non-invariant embedding descriptors do (hull area
$\rho = 0.411$), so invariance is necessary but not sufficient in
finite samples; and this ICU cohort is one observation-operator
pair, so replication is required. A Mann--Whitney test of the
predicted class separation (invariants $>$ observation-coupled,
after heart-rate control) does not reject the null
($U = 12.0$, $p = 0.54$): on BIDMC the clean class partition is
therefore \emph{not} supported, and $\lmax$ transfer is the single
effect that survives heart-rate control. We report this null rather
than the earlier pooled estimate.
\end{remark}

\subsubsection{The Variational Domain \texorpdfstring{$\cV$}{V}}
\label{sec:var}
\begin{definition}[FTLE Field / $\TV$]
\label{def:TV}
The \emph{Variational mathematical transformation}
$\TV : \R^{N \times d} \times \N \to \R^N$
maps the trajectory matrix and forward-time horizon
$\Delta n$ to the finite-time Lyapunov exponent (FTLE)
field $\bm{\lambda} \in \R^N$, where entry $\lambda_n$
is the local expansion rate at reconstructed state
$\Xmat_n$.
For each state $\Xmat_n$, let $\mathcal{N}_k(n)$ denote
its $k$ nearest neighbors with temporal separation
$> \lfloor\bar{T}_{\mathrm{beat}}/2\rfloor$. The local
deformation gradient is:
\begin{equation}
\begin{aligned}
  \hat{\Fmat}_n = \arg\min_{\Fmat}
  \sum_{j \in \mathcal{N}_k(n)}
  \bigl\| &(\Xmat_{j+\Delta n}-\Xmat_{n+\Delta n}) \\
           &-\Fmat(\Xmat_j-\Xmat_n)\bigr\|_2^2 .
\end{aligned}
\label{eq:Fhat}
\end{equation}
This fits a local linear map between neighbor displacement vectors around $\Xmat_n$ and their forward-evolved displacements around $\Xmat_{n+\Delta n}$, thereby separating local deformation from translation of the reference trajectory. The right Cauchy-Green
tensor $\hat{\Cmat}_n = \hat{\Fmat}_n^\top \hat{\Fmat}_n$
is symmetric positive semi-definite and measures the
net squared stretching in every direction
\cite{haller2000,shadden2005}; the FTLE
extracts its maximum growth rate:
\begin{equation}
  \lambda_n = \frac{1}{2\,\Delta n}
    \ln\lmax(\hat{\Cmat}_n).
  \label{eq:ftle}
\end{equation}
Under Condition~1, where local differentiability is required for this construction, $\hat{\Fmat}_n$
can approximate the local derivative of the finite-time flow map when neighborhoods shrink appropriately, local sampling is sufficient, and measurement noise is controlled. Under those additional conditions, $\lambda_n$ approximates local finite-time expansion; differentiability of $f$ alone is not a consistency guarantee. If $f$ is not
differentiable in a neighbourhood of $\Xmat_n$, for
instance at states corresponding to non-smooth transitions
in the underlying dynamics, $\hat{\Fmat}_n$ need not
approximate any well-defined local linearisation, and
$\lambda_n$ should not be interpreted as a local expansion
rate at that state.
\end{definition}
\begin{definition}[Operational Cardiac Phase Regions]
\label{def:phases}
To obtain phase-conditioned summaries from the observed waveform, we define two operational regions using discrete first and second differences of $x$:
\begin{align}
  \Osys &= \{n : \Delta x_n > 0\text{ and }
    \Delta^2 x_n > 0\},\label{eq:Osys}\\
  \Odia &= \{n : \Delta x_n < 0\text{ and }
    \Delta^2 x_n < 0\},\label{eq:Odia}
\end{align}
where $\Delta x_n = x_{n+1} - x_n$ and
$\Delta^2 x_n = x_{n+2} - 2x_{n+1} + x_n$.
$\Osys$ selects accelerating samples on the rising portion of the pulse waveform, which are associated with the systolic upstroke; $\Odia$ selects falling, concave-down samples associated with an early portion of diastolic decay. These sets are signal-defined phase proxies rather than direct labels of the full mechanical systolic and diastolic intervals. The remaining falling-waveform region is excluded in the present implementation because the associated BP relation was weaker in the motivating AVCT analysis \cite{oladunni2026avct}; this exclusion is empirical rather than a universal physiological boundary.
\end{definition}
\paragraph{Phase-Aggregated Expansion Rates.}
\begin{align}
  E_{\mathrm{sys}} &= \frac{1}{\abs{\Osys}}
    \sum_{n \in \Osys} \lambda_n,\label{eq:Esys}\\
  E_{\mathrm{dia}} &= \frac{1}{\abs{\Odia}}
    \sum_{n \in \Odia}\abs{\lambda_n}.\label{eq:Edia}
\end{align}
$E_{\mathrm{sys}}$ is the average FTLE over the operational systolic-upstroke region, and $E_{\mathrm{dia}}$ is the average absolute FTLE over the operational early-diastolic-decay region. In the motivating AVCT analysis these quantities were associated with BP-related dynamics \cite{westerhof2009,oladunni2026avct}. They are therefore treated here as phase-conditioned deformation descriptors; direct identification with contractile force, vascular resistance, or compliance requires separate physiological validation.
\begin{proposition}[Local Variational-Domain Calibration Model]
\label{thm:bp}
Let $E_{\mathrm{sys}}$ and $E_{\mathrm{dia}}$ denote phase-aggregated finite-time divergence descriptors. Over a restricted physiological operating range, systolic and diastolic pressure may be modeled locally as
\begin{align}
  P_{\mathrm{sys}} &\approx \alpha_s E_{\mathrm{sys}}+\beta_s, \label{eq:Psys}\\
  P_{\mathrm{dia}} &\approx \alpha_d E_{\mathrm{dia}}+\beta_d. \label{eq:Pdia}
\end{align}
The coefficients $\theta_s=(\alpha_s,\beta_s)$ and $\theta_d=(\alpha_d,\beta_d)$ are calibration parameters that must be identified from sufficient reference data or additional independently justified constraints. A single paired blood-pressure reference does not, by itself, identify four free affine parameters.
\end{proposition}
\begin{proof}
Moens--Korteweg and Windkessel relations motivate monotone dependence of pulse propagation and diastolic decay on compliance and vascular resistance \cite{westerhof2009,ding2017}. Smooth physiological relationships can be approximated locally to first order, yielding affine models over a restricted range. These physical relations do not imply $\alpha_d=-\alpha_s$, because their compliance and resistance dependencies differ. Parameter identifiability therefore follows from the calibration design, not from a one-pair algebraic identity. \qed
\end{proof}

\subsection{Conditional Domain Decomposition}
\label{sec:sufficiency}
Let $y$ denote a specified cardiovascular endpoint and let $X=\TG[x]$ denote the reconstructed trajectory matrix. To preserve the distinction between mathematically transformed domain representations and the descriptors measured within them, let
\begin{equation}
  \cG=\TG[x],\qquad \cS=\TS[\cG],\qquad \cV=\TV[\cG,\Delta n],
  \label{eq:domainoperators}
\end{equation}
and define endpoint-analysis variables
\begin{equation}
  G=\psi_G(\cG),\qquad S=\psi_S(\cS),\qquad V=\psi_V(\cV),
  \label{eq:domainfeatures}
\end{equation}
where $\psi_G$, $\psi_S$, and $\psi_V$ denote selected measurable functionals or descriptor maps within the corresponding mathematically transformed representations. Thus the domains exist at the representation level before any particular descriptor set is selected.

\begin{theorem}[Conditional Domain Decomposition]
\label{thm:sufficiency}
For any jointly distributed $(X,G,S,V,y)$ with $(G,S,V)$ deterministic functions of $X$,
\begin{align}
 I(X;y) &= I(G;y)+I(S;y\mid G)+I(V;y\mid G,S) \nonumber\\
 &\quad + \delta_y, \label{eq:sufficiency}\\
 \delta_y &:= I(X;y\mid G,S,V) \ge 0. \label{eq:residual}
\end{align}
The equality is exact. Approximate endpoint-specific sufficiency of the selected three-domain representation corresponds to $\delta_y\approx0$; it is an empirical hypothesis, not a universal consequence of delay embedding.
\end{theorem}

\begin{proof}
Because $(G,S,V)$ are deterministic functions of $X$, the pair $(X,G,S,V)$ contains exactly the same information as $X$, so $I(X;y)=I(X,G,S,V;y)$. Repeated application of the chain rule gives
\begin{align*}
I(X,G,S,V;y) &= I(G;y)+I(S;y\mid G)\\
&\quad +I(V;y\mid G,S)+I(X;y\mid G,S,V),
\end{align*}
which yields \eqref{eq:sufficiency}. Non-negativity of conditional mutual information gives $\delta_y\ge0$. No strict positivity of an increment is implied by the algebra; $I(S;y\mid G)>0$ or $I(V;y\mid G,S)>0$ is an endpoint-specific empirical statement to be tested. \qed
\end{proof}

\noindent\emph{Interpretation.} The decomposition separates what the selected geometric descriptors predict, what the selected ergodic descriptors add beyond geometry, and what the local variational descriptors add beyond both. Importantly, the conditioning variables in the incremental terms are the selected descriptor vectors $G$ and $(G,S)$, not the full reconstructed trajectory $X=\cG$. Although $G$, $S$, and $V$ are deterministic functions of $X$, conditioning on a lossy descriptor map such as $G=\psi_G(\cG)$ does not in general determine $S$ or $V$; consequently the conditional increments may be zero or positive depending on the endpoint and descriptor maps. The residual $\delta_y$ measures endpoint-relevant information in the reconstructed trajectory that is not captured by the chosen descriptors. This formulation makes ADT falsifiable: a small residual supports approximate sufficiency for a particular endpoint, whereas a large residual indicates that the selected representation is incomplete for that endpoint.

\begin{theorem}[Domain-Transformation Recovery Impossibility]
\label{thm:impossibility}
Let $T_j:\mathcal U_j\rightarrow\mathcal Z_j$ denote any ADT domain operator on its admissible input space, with $j\in\{G,S,V\}$, and let $C:\mathcal U_j\rightarrow\mathcal C$ be a dynamical property defined on that same input space. If there exist $u_1,u_2\in\mathcal U_j$ such that
\begin{equation}
T_j(u_1)=T_j(u_2)\quad\text{but}\quad C(u_1)\neq C(u_2),
\label{eq:impossibility}
\end{equation}
then no function $g:\mathcal Z_j\rightarrow\mathcal C$ can satisfy $C=g\circ T_j$ on all of $\mathcal U_j$. Hence, once $T_j$ collapses a distinction on which $C$ depends, $C$ is not exactly recoverable from the resulting domain representation alone.
\end{theorem}

\begin{proof}
Assume, for contradiction, that such a function $g$ exists. From $T_j(u_1)=T_j(u_2)$,
\[
C(u_1)=g(T_j(u_1))=g(T_j(u_2))=C(u_2),
\]
which contradicts $C(u_1)\neq C(u_2)$. Therefore no such $g$ exists. \qed
\end{proof}

\begin{remark}[Scope of the Impossibility Result]
Theorem~\ref{thm:impossibility} is a property-specific statement about information collapsed by a particular domain transformation. It does not imply that $\TG$, $\TS$, or $\TV$ is globally non-invertible, that the three ADT domains are mutually exclusive, or that they form an exhaustive decomposition of all dynamical information. Rather, it states the precise condition under which no downstream descriptor or estimator operating only on $T_j(u)$ can reconstruct a property that varies within a fiber of $T_j$.
\end{remark}

\begin{remark}[Role of Delay Embedding]\label{rem:domainlimits}
In the ideal noiseless setting and under the genericity, smoothness, and dimension conditions of delay-embedding theory \cite{takens1981,sauer1991}, the reconstructed trajectory can be diffeomorphic to the underlying attractor. This licenses analysis of attractor-derived quantities but does not imply that a finite list of descriptors is sufficient for every clinical target. With finite sampling and acquisition noise, reconstruction error must be assessed empirically.
\end{remark}

\begin{remark}[Information-Decomposition Analogy]
Equation~\eqref{eq:sufficiency} is a chain-rule decomposition of endpoint information, not an isometric Parseval identity and not a claim that the transformations are orthogonal. Its value is operational: conditional information increments and $\delta_y$ can be estimated to test whether additional domains contribute to a specified endpoint.
\end{remark}
\begin{corollary}[Endpoint-Specific Domain Contribution]
\label{cor:nonredundant}
For a specified endpoint $y$, $S$ contributes information beyond $G$ exactly when $I(S;y\mid G)>0$, and $V$ contributes beyond $(G,S)$ exactly when $I(V;y\mid G,S)>0$. Cross-observation agreement of a particular descriptor requires a separate invariance or recoverability argument; it does not follow from the decomposition alone.
\end{corollary}
\begin{proof}
The first two statements are definitions of positive conditional information contribution. The final statement follows because two observation functions may produce diffeomorphic ideal reconstructions while finite-scale, metric-dependent, partition-dependent, or noise-sensitive descriptors can still differ. Lemma~\ref{lem:invariance} identifies examples of both invariant and observation-dependent quantities. \qed
\end{proof}

\begin{proposition}[Operational Scope of the Three Domain Representations]
\label{prop:minmax}
ADT defines three linked representations through explicit mathematical transformations: (i) the Geometry Domain $\cG=\TG[x]$, which contains the reconstructed orbit representation; (ii) the Ergodic Domain $\cS=\TS[\cG]$, which exposes selected long-horizon or asymptotic organization; and (iii) the Variational Domain $\cV=\TV[\cG,\Delta n]$, which exposes localized finite-time deformation. Domain-specific descriptors are functionals measured within these representations rather than the objects used to define the domains. Within ADT, the three representations are distinguished by operator input, temporal scope, codomain, and the dynamical question interrogated: reconstruction of orbit geometry ($\cG$), long-horizon/asymptotic organization ($\cS$), and localized finite-time deformation ($\cV$). They may yield different endpoint-specific information increments. This proposition establishes the scope of the proposed architecture only; it is not a theorem that three is a mathematically minimal, unique, natural, or exhaustive number of domains. Alternative operators or representation architectures may be appropriate for other dynamical questions.
\end{proposition}

\subsection{Is ADT Feature Grouping? The Conceptual Distinction}
ADT differs from retrospective feature grouping because its categories are specified by operators before endpoint descriptors are chosen. In the present architecture, $\TG$, $\TS$, and $\TV$ have different inputs/temporal scopes and produce representations for different dynamical questions, whereas feature grouping merely categorizes pre-existing features. $\TG$ constructs the reconstructed state-space representation; $\TS$ maps that reconstruction to selected long-horizon or asymptotic representations; and $\TV$ maps it to localized finite-time deformation representations. Their endpoint-specific descriptor contributions are tested through Theorem~\ref{thm:sufficiency}, while Theorem~\ref{thm:impossibility} states the general recovery consequence when any specified ADT transformation collapses a property-relevant distinction. These results give the architecture testable information-accounting and recoverability consequences, but they do not prove that the three-domain choice is uniquely privileged. Table~\ref{tab:adt-vs-grouping} summarizes the distinction.

\begin{table}[!tb]
\centering
\caption{ADT Transformations vs Feature Grouping: A Conceptual Distinction}
\label{tab:adt-vs-grouping}
\footnotesize
\renewcommand{\arraystretch}{1.12}
\begin{tabular}{P{1.6cm}P{3.0cm}P{3.0cm}}
\toprule
\textbf{Aspect} & \textbf{Feature Grouping} & \textbf{ADT Transformations} \\
\midrule
\textbf{Operation} & Categorize existing features & Transform signal through mathematical transformations ($\TG$, $\TS$, $\TV$) [Def.~\ref{def:traj}, \ref{def:TS}, \ref{def:TV}] \\
\addlinespace
\textbf{Math.} & No required transformation & Explicit mathematical-transformation definitions: \\
\textbf{Just.} & categorisation & $\TG$, $\TS$, $\TV$ \\
\addlinespace
\textbf{Output} & Subsets of original feature space & Operational representations: orbit geometry, ergodic organization, finite-time deformation \\
\addlinespace
\textbf{Scope} & Multiple valid & Three linked representations; no claim of \\
& groupings possible & uniqueness or exhaustiveness \\
\addlinespace
\textbf{Evaluation} & Convenience, & Descriptor contributions tested through \\
& intuition, domain knowledge & conditional information with explicit residual \\
\bottomrule
\end{tabular}
\end{table}

\section{Geometry-Anchored ADT Instantiation: \texorpdfstring{$\SCSI$}{SCSI}}
\label{sec:instantiation}
The $\SCSI$ (Cardiovascular Stability Index) formulation is a geometry-anchored wearable-PPG instantiation of ADT: it first applies $\TG$ to construct the trajectory matrix $\Xmat$, then computes a mixed finite descriptor set from $\Xmat$ and the observed waveform. Domain membership is assigned by the dynamical property represented, not merely by the fact that a quantity is computed from $\Xmat$. Thus asymptotic LLE and long-horizon recurrence, invariant-measure, and statistical summaries are associated with the Ergodic Domain $\cS$; localized finite-time deformation is associated with the Variational Domain $\cV$; and reconstructed orbit geometry and geometry-derived complexity descriptors are associated with the Geometry Domain $\cG$. The present $\SCSI$ implementation does not explicitly construct the full $\TS$ or $\TV$ transformations; it uses practical finite proxies anchored in the Geometry reconstruction. The empirical results in Sections~\ref{sec:protocol}--\ref{sec:validation} therefore evaluate this geometry-anchored ADT instantiation rather than validating the full three-transformation theory.

\subsection{PPG Observability Extension}
\label{sec:prop1}
The three-component CSI formula \eqref{eq:csi} was derived in
\cite{oladunni2026cst} for ECG. Instantiating the geometry-anchored wearable PPG representation requires three structural extensions,
which we formalise as a proposition.
\begin{proposition}[PPG Observability Extension]
\label{prop:ppg}
The present application of $\TG$ to a wearable PPG observable
$x_n = h_{\mathrm{PPG}}(z_n, p_n) + \varepsilon_n$ uses
three classes of modification relative to the ECG instantiation:
\emph{(i) Geometry Domain quality control} (Condition~2). Motion
artifact and poor optical contact can invalidate the assumed fixed,
smooth observation model and render the resulting trajectory
matrix $\Xmat$ unreliable. A signal quality gate
$Q_w \in [0,1]$ and homeostasis proxy
$\Omega_w = 1 - \mathbf{1}[\abs{\hat{z}_w} > 3]$ act as implementation-specific quality and operating-envelope gates, rejecting windows for which downstream reconstruction-derived estimates are treated as unreliable. They do not directly enforce the density-gradient boundary condition in Eq.~\eqref{eq:boundary}.
\emph{(ii) Peripheral-robust substitution} (Condition~3). The
peripheral state $p_n$ (arterial compliance, resistance,
venous return) disrupts the Geometry Domain estimate of
$R_{\mathrm{det}}$ from $\Xmat$ (empirically: $\rho_{R_{\mathrm{det}}}
= 0.123$ cross-modally, CI spanning zero \cite{oladunni2026cst}). Two
Geometry-domain complexity estimators designed to be less sensitive to smooth
amplitude scaling by $p_n$ replace this poorly transferring Ergodic-domain descriptor:
Higuchi Fractal Dimension
$D_{f,w}$ \cite{higuchi1988}, a finite-sample time-series fractal-complexity estimator; and spectral
energy concentration $E_w$, measuring energy at the
fundamental cardiac frequency.
\emph{(iii) Peripheral state characterisation} (Condition~4).
Selected peripheral-state effects not represented by $(\lmax, H_{\mathrm{stat}})$ are approximated by three operational proxies: autonomic tone $A_w$
(sympatho-vagal balance from low-frequency/high-frequency (LF/HF) spectral ratio), vascular
compliance $B_w$ (PPG amplitude and contour), and recovery
capacity $R_w$ (post-perturbation dynamics). Together these
extend the Geometry Domain feature vector to cover peripheral
cardiovascular state.
\end{proposition}
\noindent The present wearable-PPG instantiation uses three extension classes: Geometry-Domain quality control, substitution of a poorly transferring recurrence descriptor by Geometry-computed estimators, and operational characterization of selected peripheral-state effects.

\subsection{Feature Vector and CSI Formula}
\label{sec:features}
The nine-dimensional feature vector $\mathbf{u}_w$ collects
all finite SCSI features at window scale $w$:
\begin{equation}
  \mathbf{u}_w = [Q_w,\; \Omega_w,\; H_w,\; D_{f,w},\;
    E_w,\; \Lambda_w,\; A_w,\; B_w,\; R_w],
  \label{eq:uvec}
\end{equation}
where $\Lambda_w = e^{-\lambda_w}$ maps the stabilized largest Lyapunov exponent (LLE)
to $[0,1]$. Because $H_w$, $D_{f,w}$, $E_w$, and $\Lambda_w$
are all estimates of attractor geometric complexity under
the endpoint-specific bounded-complexity hypothesis, they are composited into the nonlinear complexity
module:
\begin{equation}
  C_{\mathrm{NL},w} = G_{\mathrm{NL},w}
    \sum_{k \in \{\mathrm{SE,HFD,LLE,E}\}}
    w^*_k\,\psi_k,
  \label{eq:cnl}
\end{equation}
where $\psi_k = \exp(-(f_k - \mu_{f_k})^2 / 2\sigma_{f_k}^2)$
is the bounded-optimality kernel rewarding values near the
population median under the endpoint-specific bounded-complexity hypothesis, and $G_{\mathrm{NL},w}$ gates out
segments with extreme z-score deviation. The six-component
$\SCSI$ is then:
\begin{equation}
  \SCSIw = G_w \sum_{j=1}^{6}
    \beta^*_j X_{j,w},
    \quad \SCSIw \in [0,1],
  \label{eq:scsi}
\end{equation}
where $\mathbf{X}_w = [C_{\mathrm{NL},w}, A_w, \Omega_w,
Q_w, R_w, B_w]^\top$ and $G_w = Q_w \cdot G_{\mathrm{NL},w}$
is the composite quality gate. Boundedness follows from
$X_{j,w} \in [0,1]$ and $\beta^*_j \ge 0$,
$\sum_j \beta^*_j = 1$ (Proposition~\ref{thm:stability}).
\textbf{Stabilized LLE Estimator.}
The classical Rosenstein estimator \cite{rosenstein1993}
requires that nearest neighbors in the trajectory matrix
are not temporal neighbors. The temporal-separation condition
$\abs{m_j - n} > \lfloor\bar{T}_{\mathrm{beat}}/2\rfloor$
acts as a Theiler-style exclusion window: it prevents temporally adjacent samples
from being selected as nearest neighbors when estimating divergence. It does not,
by itself, establish injectivity of the delay embedding; injectivity depends on the
embedding conditions stated in Section~\ref{sec:domains}. Without adequate temporal
separation, nearest-neighbor divergence estimates can be biased by serial proximity. The stabilized estimator maps:
\begin{equation}
  \Lambda_w = e^{-\lambda_w},
  \quad \lambda_w \uparrow \;\Rightarrow\; \Lambda_w \downarrow,
  \label{eq:lambda}
\end{equation}
so that higher estimated trajectory divergence yields lower $\Lambda_w$; its contribution to $\SCSI$ is then shaped by the bounded-optimality kernel rather than interpreted as monotonically healthier.
\begin{algorithm}[t]
\caption{$\SCSI$: Geometry-Anchored ADT Instantiation}
\label{alg:scsi}
\small
\begin{algorithmic}[1]
\Require $x_n \!\in\! \R^W$, $W\!=\!128$, $f_s\!=\!30$\,Hz;
  $\bm{\theta}^* \!=\! (m^*, \tau^*, r^*, k_{\max}^*, \theta^*)$;
  $\{\mu_k, \sigma_k\}$; $\{w^*_k\}$, $\{\beta^*_j\}$
\Ensure $\SCSIw \in [0,1]$ or \textsc{reject}
\State $Q_w \gets \mathrm{SignalQuality}(x_n)$; $\Omega_w \gets \mathbf{1}[|\hat{z}_n| \le 3]$
\If{$Q_w < \theta^*$ \textbf{or} $\Omega_w = 0$} \Return \textsc{reject} \EndIf
\State $\Xmat \gets \mathrm{DelayEmbed}(x_n, m^*, \tau^*)$ \Comment{$\TG$: $\Xmat \in \R^{N \times m^*}$}
\State $H_w \gets \mathrm{SampEn}(\Xmat, m^*, \tau^*, r^*)$
\State $D_{f,w} \gets \mathrm{HFD}(\Xmat, k_{\max}^*)$
\State $\lambda_w \gets \mathrm{LLE\_Stab}(\Xmat, \lfloor\bar{T}_{\mathrm{beat}}/2\rfloor)$; $\Lambda_w \gets e^{-\lambda_w}$
\State $E_w \gets \mathrm{SpectralEnergy}(x_n, f_s)$
\State $\psi_k \gets \exp(-(f_k \!-\! \mu_{f_k})^2/2\sigma_{f_k}^2)$, $k \!\in\! \{\mathrm{SE,HFD,LLE,E}\}$
\State $C_{\mathrm{NL},w} \gets G_{\mathrm{NL},w} \sum\nolimits_k w_k^* \psi_k$
\State $A_w \gets \mathrm{AutonomicProxy}(x_n, f_s)$; $B_w \gets \mathrm{VascularProxy}(x_n)$
\State $R_w \gets \mathrm{RecoveryProxy}(x_n)$
\State $G_w \gets Q_w \cdot G_{\mathrm{NL},w}$; $\mathbf{X}_w \gets [C_{\mathrm{NL},w}, A_w, \Omega_w, Q_w, R_w, B_w]^\top$
\State \Return $G_w \sum_{j=1}^{6} \beta^*_j X_{j,w}$
\end{algorithmic}
\end{algorithm}
\noindent
Algorithm~\ref{alg:scsi} instantiates $\TG$ explicitly. The initial steps perform signal-quality control and construct the Geometry-Domain reconstruction; subsequent steps compute mixed finite descriptors and selected peripheral-state proxies from $\Xmat$ and $x_n$.
The transition matrix $\Pmat$ and FTLE field $\bm{\lambda}$
of $\TS$ and $\TV$ are never constructed; full Ergodic and
Variational Domain instantiation is future work.
\subsection{Bayesian Parameter Optimisation}
\label{sec:bayesian}
Prediction~2 of ADT anticipates that a within-beat $\cG$-dominant tachypnea endpoint will favor short windows (high temporal resolution) and moderate-to-high embedding dimension (capturing local geometry). Multivariate Bayesian optimisation (Optuna v3, Tree-structured Parzen Estimator \cite{watanabe2023}, 300 trials) yielded a configuration consistent with this prediction: optimal configuration is $W^* = 128$ samples (4.3\,s, shortest available), $m^* = 8$, $\tau^* = 7$, $r^* = 0.116\sigma$, quality gate $\theta^* = 0.976$, with convergence at trial~82.

The optimised embedding dimension $m^*=8$ is compatible with the empirical effective-dimension estimate near $3.5$ under the conservative reconstruction heuristic used in the development search. This compatibility is not a direct application or validation of Whitney's manifold embedding theorem. Component weights: SampEn 0.431, HFD
0.483, LLE 0.043, spectral energy 0.043 in $C_{NL}$; $\SCSI$ weights
($C_{NL}$:0.260, $A$:0.210, $\Omega$:0.198, $Q$:0.267, $R$:0.048,
$B$:0.017). This optimized configuration is consistent with Prediction~2 and achieved the best AUC within the prespecified development search space.
\section{Corrected Evaluation Protocol}
\label{sec:protocol}
Evaluating wearable cardiovascular complexity indices requires
a protocol that prevents three systematic artifacts identified
in this work. Each artifact is a consequence of violating a
statistical independence assumption; each correction is stated
precisely.
\textbf{Artifact~1 (Segment-level cross-validation (CV) leakage, $+0.062$ AUC).}
Shuffling windows without subject stratification places
temporally adjacent segments into both train and validation
folds. Correction: record-level GroupKFold.
\textbf{Artifact~2 (Test-set normalisation leakage,
$+0.308$ AUC).} z-score statistics computed across all records
including held-out test records. The prior unoptimised
heuristic index, the Cardiac Parametrization Stability (CPS)
index, the predecessor of the $\SCSI$, has an
AUC of 0.881 (leaked) that falls to 0.573 (corrected),
fully explaining the previously reported DeLong finding
($p = 0.015$). Correction: training-only normalisation.
\textbf{Artifact~3 (Pooled AUC overweighting).} Pooled AUC
weights each segment equally; records with more segments
dominate. A 1D convolutional neural network (CNN) achieves pooled AUC~$= 0.804$ but
per-record mean AUC~$= 0.380 \pm 0.117$ (below chance).
Correction: report both pooled and per-record AUC.
Table~\ref{tab:artifacts} summarises all three artifacts.
The two biases are quantified on different evaluations and
are therefore not additive on a single figure: Artifact~1
($+0.062$) is measured on the development set (segment- vs.\
record-level CV), whereas Artifact~2 ($+0.308$) is measured
on the held-out test set, where leaked normalisation raises
the CPS AUC to 0.881 against a fair-normalisation value of
0.573. Taken together, the reported CPS AUC of 0.881
represents $1.54\times$ the fully unbiased performance
of 0.573.
\begin{table}[H]
\centering
\caption{Evaluation Artifact Summary}
\label{tab:artifacts}
\renewcommand{\arraystretch}{1.1}
\setlength{\tabcolsep}{3pt}
\begin{tabular}{P{0.35cm} P{2.6cm} P{1.7cm} P{2.6cm}}
\toprule
\textbf{\#} & \textbf{Artifact} & \textbf{Inflation}
  & \textbf{Correction} \\
\midrule
1 & Segment-level CV leakage
  & $+0.062$
  & Record-level GroupKFold \\
2 & Test-set normalisation leakage
  & $+0.308$
  & Train-only normalisation \\
3 & Pooled AUC overweighting
  & $+0.424$$^\dagger$
  & Per-record AUC \\
\midrule
\bottomrule
\multicolumn{4}{p{7.8cm}}{\small $^\dagger$ CNN: pooled 0.804, per-record $0.380\pm0.117$.}
\end{tabular}
\end{table}

\section{Empirical Validation}
\label{sec:validation}

\subsection{Datasets and Multiscale Empirical Evaluation}
Four heterogeneous PPG datasets spanning clinical, ambulatory,
and consumer contexts provide a broad distributional evaluation of the geometry-anchored implementation across
176{,}742 segments at three temporal scales
($w \in \{256, 512, 1024\}$ samples; 8.5, 17.1, 34.1\,s
at $f_s = 30$\,Hz). The \emph{Beth Israel Deaconess Medical
Center (BIDMC) PPG and Respiration
Dataset} \cite{pimentel2017,goldberger2000physiobank} provides 53 adult intensive care unit (ICU) records
with simultaneous capnographic respiratory rate (RR); the
\emph{BUT-PPG database} \cite{nemcova2021} provides
laboratory-controlled smartphone camera recordings; the
\emph{RWS-PPG dataset} \cite{jokic2026} provides large-scale
real-world smartphone PPG recordings; and the \emph{Welltory dataset} \cite{neshitov2021}
provides consumer smartphone recordings. For clinical
optimisation and held-out testing, BIDMC records were split
patient-level into 35 development and 18 held-out test records.
We also examined the \emph{CapnoBase dataset}
\cite{karlen2013} (42 elective-surgery recordings, zero
training overlap) for external validation. The clinical endpoint was tachypnea
(RR $>$ 25\,bpm, 3{,}044 labelled segments, 7.1\% tachypneic).
\paragraph{Geometry Domain Enforcement.}
The stabilized LLE estimator demonstrates monotonically increasing
mean finite-window LLE estimates across temporal scales
($0.086 \to 0.104 \to 0.123$), illustrating sensitivity to reconstruction
scale. This empirical pattern does not directly test the spectral-gap subspace
perturbation bound in Proposition~\ref{prop:artifact}.
Although LLE estimates an Ergodic Domain quantity ($\lmax$),
its computation here is a finite-sample operation on $\Xmat$
directly (a $\cG$ operation) without ever constructing
$\Pmat$ or $\mu$ explicitly, consistent with $\SCSI$
instantiating $\TG$ only. This asymptotic $\lmax$ (an Ergodic
Domain invariant) is distinct from the finite-time Lyapunov field
of the Variational Domain; the two are different objects despite
the shared Lyapunov terminology.
The quality gates $Q_w$ and $\Omega_w$ are clinical
safety constraints: Proposition~\ref{prop:artifact} shows that reconstruction stability is not guaranteed when the signal-subspace separation condition fails. A failed quality gate should therefore surface as an explicit signal-quality warning rather than being treated as an equally reliable reading.

\paragraph{Cross-Dataset Distributional Differences.}
Table~\ref{tab:crossdata} summarises Kruskal-Wallis and
pairwise Mann-Whitney distributional differences across four sources. Because the
datasets differ in acquisition hardware, populations, protocols, and recording
conditions, these contrasts are not interpreted as pure physiologic effects.
Cross-scale consistency $\kappa > 0.97$ shows stable $\SCSI$ ranking across the tested temporal scales. This empirical stability of a mixed finite composite should not be conflated with mathematical invariance of any individual Ergodic-Domain descriptor.
$\SCSI$ is positively correlated with respiratory rate across
53 BIDMC records (Spearman $r = 0.346$,
95\,\% confidence interval (CI) $[0.035, 0.604]$, $p = 0.011$) and uncorrelated
with heart rate ($r = -0.089$, $p = 0.524$). This pooled
correlation reflects between-subject structure and does not
translate into subject-level predictive power: under
subject-blocked cross-validation (GroupKFold, $4{,}833$ windows,
53 subjects) the attractor representation predicts RR only
marginally above a mean baseline (ECG+PPG mean absolute error
$2.12$ vs.\ $2.20$\,bpm, $R^2 \approx 0$), and $\lmax$ alone
does not carry RR ($R^2 < 0$).
\begin{table}[H]
\centering
\caption{Cross-Dataset Distributional Differences
  ($w = 256$ samples, Bonferroni-corrected)}
\label{tab:crossdata}
\renewcommand{\arraystretch}{1.1}
\begin{tabular}{lrr}
\toprule
\textbf{Comparison} & \textbf{Effect size} $|r|$
  & \textbf{Verdict} \\
\midrule
\multicolumn{3}{l}{\textit{Global: Kruskal-Wallis
  $H = 23{,}415$, $\eta^2 = 0.351$, $p < 0.001$}} \\
\addlinespace
BIDMC vs.\ BUT-PPG   & 0.612 & Large \\
BIDMC vs.\ RWS-PPG   & 0.743 & Large \\
BIDMC vs.\ Welltory  & 0.501 & Large \\
BUT-PPG vs.\ RWS-PPG & 0.827 & Large \\
BUT-PPG vs.\ Welltory & 0.353 & Medium \\
RWS-PPG vs.\ Welltory & 0.988 & Near-complete \\
\bottomrule
\multicolumn{3}{l}{\small All $p < 0.001$ after Bonferroni correction.}
\end{tabular}
\end{table}

\subsection{Held-Out Test}
Table~\ref{tab:results} presents $\SCSI$ and CNN performance
under the corrected protocol. At the observed 7.1\% tachypnea prevalence, the selected Youden threshold
yielded an NPV of 0.966 on the held-out test set. Because NPV is prevalence-dependent
and the held-out cohort contains only 18 records, this result is reported as test-set
performance rather than as a clinically validated rule-out claim. The held-out pooled AUC of 0.757 exceeds
the unbiased baseline of 0.573.
The CNN pooled AUC advantage (0.804 vs.\ 0.757) reverses
per-record (0.380 vs.\ 0.497), demonstrating Artifact~3.
External evaluation on CapnoBase was uninformative for
this endpoint: general anaesthesia suppresses respiratory-rate
variability, leaving only 4 of 42 records with evaluable
tachypnea labels. Cross-dataset evaluation on an
RR-variable cohort therefore remains future work.
\begin{table}[H]
\centering
\footnotesize
\caption{$\SCSI$ vs.\ 1D CNN (unbiased baseline AUC~$= 0.573$)}
\label{tab:results}
\renewcommand{\arraystretch}{1.12}
\setlength{\tabcolsep}{3pt}
\begin{tabular}{lcc}
\toprule
\textbf{Metric} & $\boldsymbol{\SCSI}$ & \textbf{1D CNN} \\
\midrule
Pooled AUC        & 0.757 & 0.804 \\
\quad 95\% CI     & [0.686--0.828] & [0.749--0.852] \\
Per-record AUC    & $0.497 \pm 0.207$ & $0.380 \pm 0.117$ \\
NPV / Specificity & 0.966 / 0.834 & -- \\
Trainable model parameters at inference & 0 & 25{,}793 \\
Wilcoxon          & -- & $p\!=\!0.129$ (ns) \\
\bottomrule
\multicolumn{3}{p{7.8cm}}{\scriptsize
  $\SCSI$ has no trainable model parameters at inference, but uses tuned scalar hyperparameters and weights reported in Section~\ref{sec:bayesian}; the CNN pooled
  advantage reverses per-record, illustrating Artifact~3.}
\end{tabular}
\end{table}

\subsection{Endpoint-Specific Domain Contribution Test: Ablation}
Table~\ref{tab:ablation} presents the full ablation. The
full-model baseline here is the 5-fold cross-validated AUC of
0.720 on the 35 development records, which differs from the
held-out pooled AUC of 0.757 (Table~\ref{tab:results}) because
the two are computed on different record sets and evaluation
protocols; the ablation deltas are interpreted relative to the
0.720 development baseline. Each $\Delta$AUC is a predictive ablation effect and is not equal to a mutual-information term in \eqref{eq:sufficiency}. The ablation therefore assesses predictive dependence within the implemented SCSI construction; direct estimates of $I(S;y\mid G)$, $I(V;y\mid G,S)$, and $\delta_y$ remain separate tests. Three findings emerge. \emph{(1) $C_{\mathrm{NL}}$ dominance}: removal collapses AUC to 0.307 ($\Delta\mathrm{AUC} = -0.413$), identifying the nonlinear-complexity composite as the dominant predictive component; SampEn (43.1\%) and HFD (48.3\%) account for 91.4\% of $C_{\mathrm{NL}}$ weight. LLE contributes only 4.3\%, consistent with $W^*=128$
being too short for reliable global Lyapunov estimation
(Ergodic Domain, long orbits; Prediction~2). Because $C_{\mathrm{NL}}$ combines descriptors measured from more than one ADT domain representation, this result does not by itself establish a pure $\cG$ information increment. It instead supports dominance of a geometry-anchored complexity representation for tachypnea.
\emph{(2) Intra-domain redundancy}: five components improve AUC when removed, indicating predictive redundancy under this fitted construction; this does not by itself establish zero conditional mutual information. Specifically,
the auxiliary implemented components are predictively redundant given the composite $C_{\mathrm{NL}}$ under this fitted model, but not conversely.
\emph{(3) Sparse retained set}: the minimal retained set
($C_{\mathrm{NL}}$, the dominant mixed complexity composite; $A_w$, a long-horizon/statistical proxy; and $\Omega_w$, a reconstruction-quality term) is a sparse predictive set. The present SCSI ablation does not instantiate $\TV$ and therefore cannot test a Variational-domain information increment for this endpoint.
\begin{table}[H]
\centering
\caption{Ablation Study: Component Contribution to $\SCSI$
  (5-fold CV, 35 BIDMC Development Records,
  Optimised Weights)}
\label{tab:ablation}
\renewcommand{\arraystretch}{1.1}
\begin{tabular}{lrrl}
\toprule
\textbf{Component removed} & \textbf{AUC} & $\mathbf{\Delta}$\textbf{AUC}
  & \textbf{ADT domain} \\
\midrule
Full $\SCSI$ (optimised) & 0.720 & -- & -- \\
\midrule
\multicolumn{4}{l}{\textit{Critical}} \\
$C_{\mathrm{NL}}$ & 0.307 & $-0.413$ & mixed $\cG/\cS$ composite \\
SampEn & 0.444 & $-0.276$ & $\cS$-type regularity \\
\multicolumn{4}{l}{\textit{Moderate}} \\
Autonomic $A_w$ & 0.681 & $-0.040$ & $\cS$ \\
Homeostasis $\Omega_w$ & 0.710 & $-0.011$ & reconstruction gate \\
HFD & 0.719 & $-0.001$ & $\cG$ \\
\multicolumn{4}{l}{\textit{Noise: removal improves AUC}} \\
LLE & 0.723 & $+0.003$ & $\cS$ asymptotic descriptor \\
Recovery $R_w$ & 0.733 & $+0.012$ & $\cS$ (redundant) \\
Signal $Q_w$ (outer) & 0.747 & $+0.026$ & reconstruction gate \\
Energy $E_w$ & 0.747 & $+0.027$ & spectral/statistical proxy \\
Vascular $B_w$ & 0.755 & $+0.035$ & $\cS$ (redundant) \\
\bottomrule
\end{tabular}
\end{table}

\subsection{Cross-Modal Applicability: ECG Dataset}

\label{sec:ecg-validation}
We evaluated the applicability of the ADT-inspired $\SCSI$ construction on a second, independent modality: electrocardiography. To test whether $\SCSI$ features extracted from ECG data could independently perform abnormality classification, we trained a separate logistic regression model on a 5{,}000-record PTB-XL subset \cite{wagner2020} (2{,}553 normal; 2{,}447 abnormal) at 100\,Hz, split 80/20 into 4{,}000 training and 1{,}000 held-out test records, using an analogous $\SCSI$ feature construction with modality-appropriate complexity estimators and independently optimised weights and bias. This model yielded AUC = 0.719 [0.686--0.750], with sensitivity, specificity, and significance testing (bootstrap and Beta-Binomial Bayesian inference; $p < 0.001$, all metrics) reported in Table~\ref{tab:ecg-ptbxl}. Because the endpoint, feature implementation, and fitted parameters differ across modalities, this experiment demonstrates cross-modal applicability rather than invariance. Property-level invariance is evaluated separately using paired ECG--PPG observations.
\begin{table}[H]
\centering
\caption{$\SCSI$ ECG: Performance and Significance}
\label{tab:ecg-ptbxl}
\footnotesize
\renewcommand{\arraystretch}{1.12}
\begin{tabular}{lcc}
\toprule
\textbf{Metric} & \textbf{Value} & \textbf{Significance} \\
\midrule
AUC (ROC) & 0.719 [0.686--0.750] & $p<0.001$ ✓ \\
Accuracy & 0.672 [0.643--0.702] & $p<0.001$ ✓ \\
Sensitivity & 0.632 [0.588--0.676] & $p<0.001$ ✓ \\
Specificity & 0.711 [0.670--0.751] & $p<0.001$ ✓ \\
Precision & 0.676 [0.634--0.719] & -- \\
F1-Score & 0.653 [0.617--0.689] & -- \\
\addlinespace
\multicolumn{3}{l}{\textit{Methodological Reliability}} \\
Bootstrap--Bayesian $\Delta$ & $< 0.001$ & Convergence ✓ \\
Bootstrap iterations & 5000 & -- \\
\addlinespace
\multicolumn{3}{l}{\textit{Test Set Confusion (n=1{,}000)}} \\
TP / Total Positive & 309 / 489 & 63.2\% \\
TN / Total Negative & 363 / 511 & 71.1\% \\
\bottomrule
\multicolumn{3}{p{3.3cm}}{\tiny PTB-XL: 5000 records. The bootstrap CI excludes chance; reported metrics exceed chance (p<0.001).}
\end{tabular}
\end{table}

Nonparametric bootstrap resampling (5{,}000 iterations; Fig.~\ref{fig:ecg-auc-bootstrap}) assesses sampling variability of the held-out AUC. The bootstrap mean is 0.719 with SD 0.016 and an empirical 95\% CI of [0.686--0.750], showing no evident instability under the specified resampling procedure.

\begin{figure*}[t]
\centering
\begin{subfigure}[b]{0.37\textwidth}
  \centering
  \includegraphics[width=\linewidth]{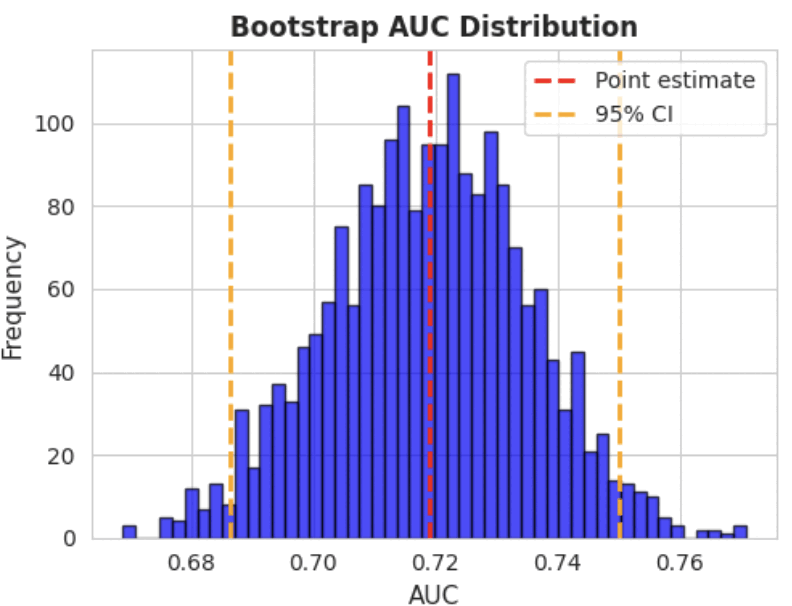}
  \caption{Bootstrap AUC distribution from 5{,}000 nonparametric resamples (mean 0.719; empirical 95\% CI [0.686--0.750]). The distribution summarizes sampling variability of the held-out AUC and shows no evident instability under the specified resampling procedure.}
  \label{fig:ecg-auc-bootstrap}
\end{subfigure}
\hfill
\begin{subfigure}[b]{0.61\textwidth}
  \centering
  \includegraphics[width=\linewidth]{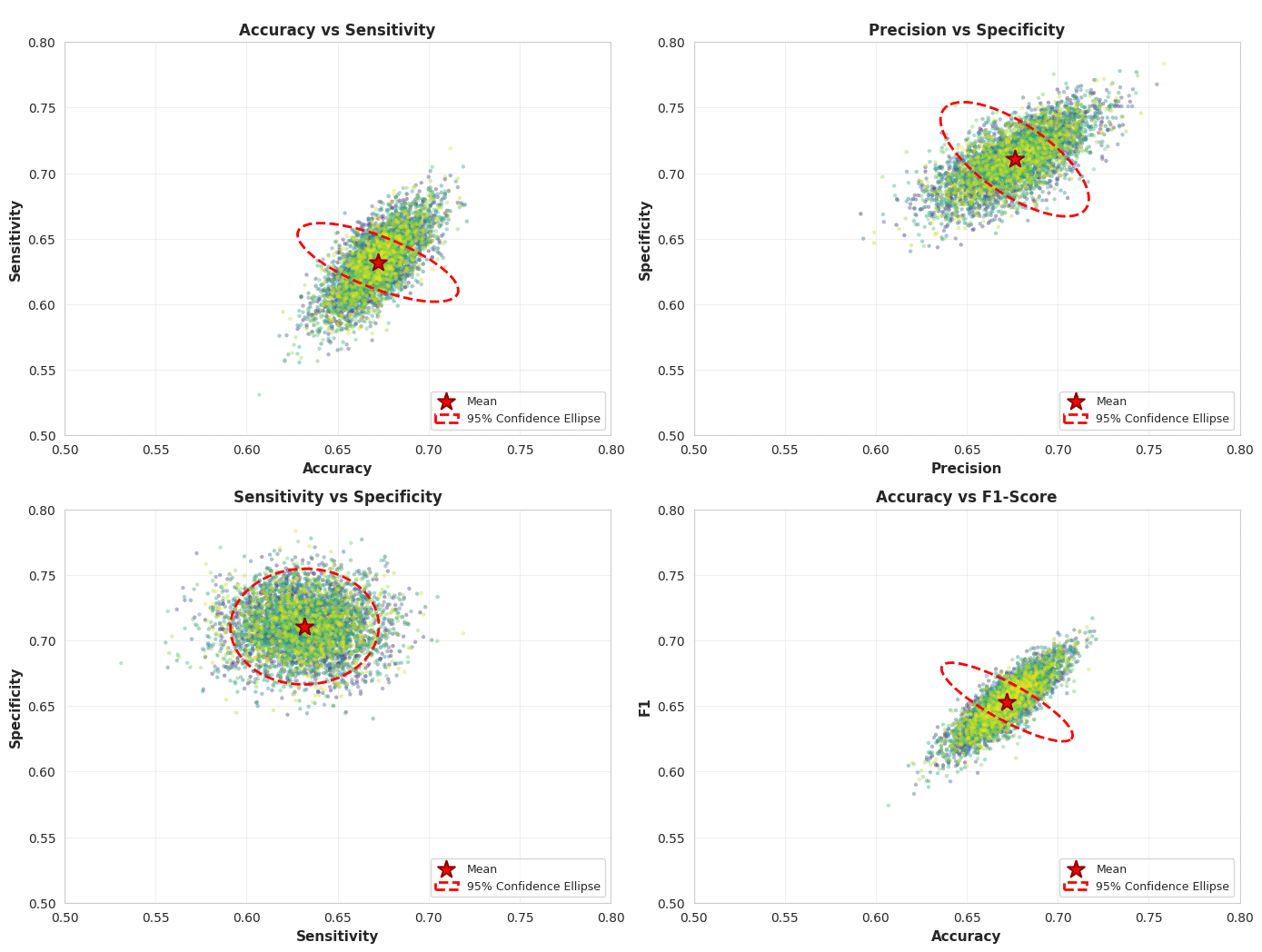}
  \caption{Bootstrap metric correlations across 5{,}000 resamples. Confidence ellipses (95\%) reveal clinical trade-offs: sensitivity gains correlate with specificity loss; the mean (red star) marks the operating point, informing threshold selection.}
  \label{fig:ecg-improvements}
\end{subfigure}
\caption{PTB-XL ECG bootstrap validation of $\SCSI$. (a) Stability of the AUC point estimate; (b) metric trade-offs across resamples.}
\label{fig:ecg-bootstrap-combined}
\end{figure*}

\paragraph{Feature Importance in ECG Domain.}

\begin{figure*}[t]
\centering
\includegraphics[width=0.74\linewidth]{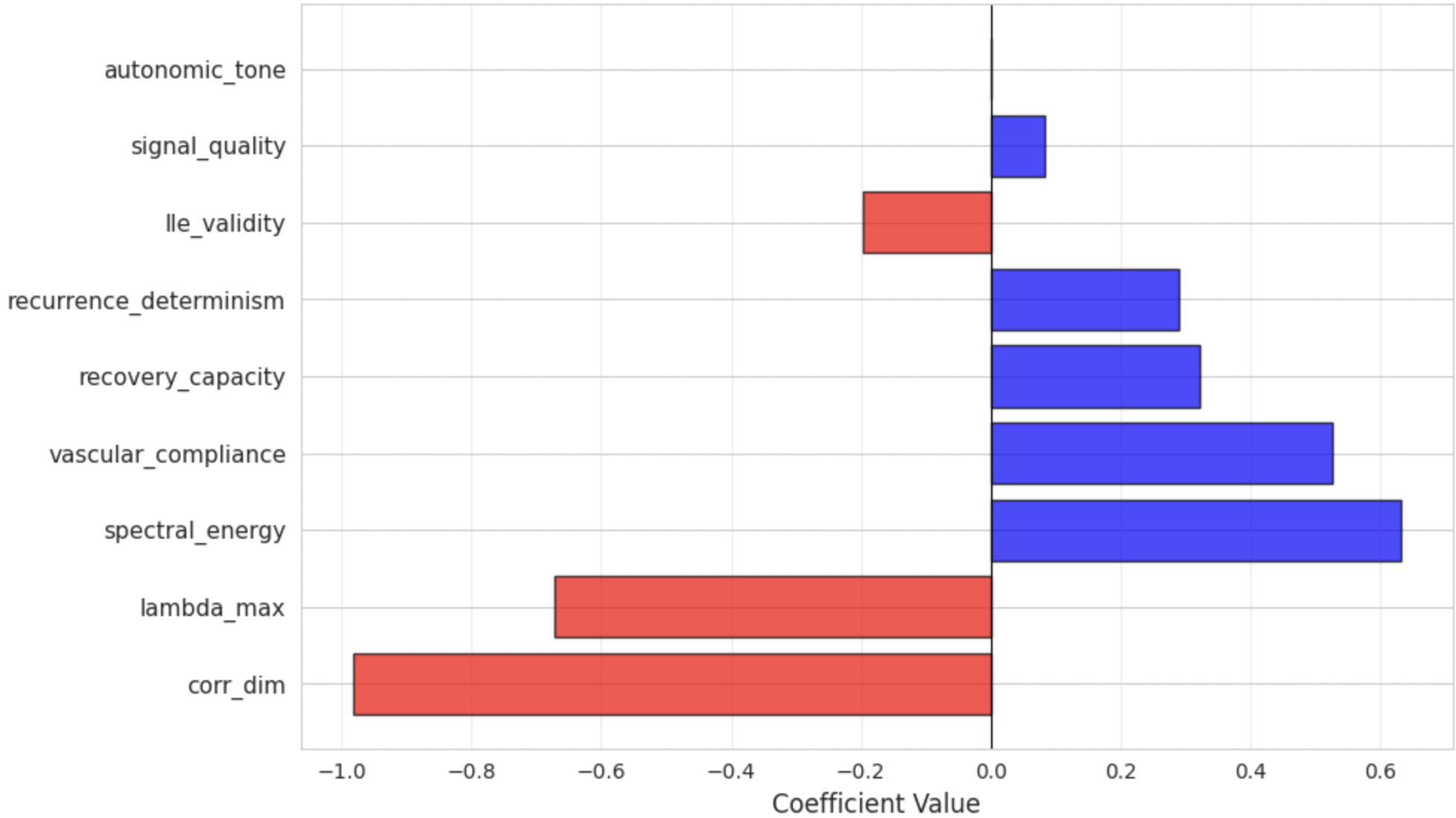}
\caption{Permutation importance of $\SCSI$ features for ECG abnormality classification (PTB-XL held-out test set, $n=1{,}000$). Positive coefficients (blue) increase abnormality risk; negative (red) are protective. The leading ECG predictors are nonlinear attractor-complexity measures, paralleling the PPG instantiation (Table~\ref{tab:ablation}) at the level of feature class rather than of identical individual features.}
\label{fig:ecg-importance}
\end{figure*}

Permutation importance on the 1{,}000-record held-out test set (Fig.~\ref{fig:ecg-importance}) shows that in both modalities the leading predictors are nonlinear attractor-complexity measures, though the specific estimators differ (sample entropy and Higuchi fractal dimension for PPG; correlation dimension and recurrence determinism for ECG). This convergence at the level of \emph{feature class} rather than identical features supports ADT's claim that attractor-based measures access cardiac dynamics reachable through multiple modalities.

Figure~\ref{fig:ecg-improvements} reveals metric trade-offs: sensitivity and specificity correlate inversely, informing clinical threshold selection where strict false-positive control (high specificity) demands accepting lower sensitivity.
The convergent PPG and ECG results, obtained with modality-appropriate attractor-complexity features rather than identical ones, support cross-modal applicability at the level of attractor-complexity feature class. They do not by themselves establish modality invariance, which is assessed separately through paired cross-observation transfer analyses.

\section{Testable Predictions}
\label{sec:predictions}
ADT yields four falsifiable predictions, summarised in
Table~\ref{tab:predictions}; each follows directly from the
domain structure and is stated precisely enough to be refuted.
\begin{table*}[!tp]
\centering
\setlength{\tabcolsep}{4pt}
\caption{Testable Predictions of Attractor Domain Theory.
  Each prediction follows directly from ADT's mathematical
  structure and is stated with sufficient specificity to
  constitute a refutable hypothesis.}
\label{tab:predictions}
\begin{tabular}{P{0.4cm} P{2.3cm} P{3.6cm} P{3.3cm} P{3.2cm}}
\toprule
\textbf{\#} & \textbf{Prediction}
  & \textbf{Statement}
  & \textbf{Status}
  & \textbf{Falsification Test} \\
\midrule
P1
& Endpoint-Domain Correspondence
& Different endpoints can favor different ADT descriptor classes; the dominant class is an empirical property of the endpoint and representation rather than a universal timescale rule.
& \textbf{Partially supported}: $C_{\mathrm{NL}}$ is the dominant SCSI ablation component ($\Delta\mathrm{AUC}=-0.413$), but its mixed $\cG/\cS$ composition prevents a pure-domain conclusion. BP remains a candidate $\cV$ endpoint when $\TV$ is explicitly instantiated.
& Estimate domain-specific conditional information with cleanly separated $G,S,V$ descriptor sets across multiple endpoints. \\
P2
& Optimal Embedding by Domain
& Within-beat $\cG$ targets require short $W^*$, high
  $m^*$ (dense geometric sampling); $\cS$ targets
  require long $W^*$, moderate $m^*$ (ergodic
  convergence).
& \textbf{Consistent with P2}: $W^*\!=\!128$, $m^*\!=\!8$
  for tachypnea (Sec.~\ref{sec:bayesian}).
& $\cS$-dominant endpoint should converge to
  $W^* \gg 128$ under Bayesian optimisation. \\
P3
& Intra-Domain Redundancy
& Descriptors measured within the same domain representation may be conditionally redundant for a specified endpoint, while descriptors from different domain representations may add information; neither relation is assumed universally.
& \textbf{Predictive redundancy observed}: 5 components improve AUC when removed (Table~\ref{tab:ablation}); direct conditional-MI estimation remains required to test the information-theoretic statement.
& Directly estimate $I(\phi_1;y\mid\phi_2)$ and compare it with predictive ablation under subject-blocked evaluation. \\
P4
& Local Variational Calibration
& Phase-localized $\cV$ descriptors may support a locally affine BP calibration over a restricted physiological range, but parameter identifiability depends on sufficient reference data or independently justified constraints.
& \textbf{Externally motivated}: AVCT reports empirical BP calibration performance \cite{oladunni2026avct}; ADT does not infer four affine parameters from a single reference pair.
& Compare one-point, multi-point, and constrained calibration designs while testing whether $\cV$ descriptors add information beyond $G$ and $S$. \\\bottomrule
\end{tabular}
\end{table*}
The present paper provides predictive-ablation evidence relevant to P1 and P3 but does not directly estimate the conditional mutual-information terms required for a domain-level confirmation. P4 is externally motivated by Attractor--Vascular Coupling Theory (AVCT) \cite{oladunni2026avct}. P2 remains a prospective hypothesis requiring testing on endpoints whose long-horizon dynamics are expected to dominate.

\section{Discussion}
\label{sec:discussion}
The formal development separates transformed representations, descriptors measured within them, and endpoint information carried by those descriptors. The Conditional Domain Decomposition is an exact chain-rule identity with an explicit residual rather than a claim of orthogonality or universal sufficiency.

\paragraph{A Late Structural Parallel with Classical Signal Representations.}
After ADT has been defined independently from reconstructed nonlinear dynamics, a limited structural parallel with classical time-, frequency-, and time--frequency representations can be noted: distinct mathematical operations expose different properties of a common observed process. This parallel is interpretive only; $\cG$, $\cS$, and $\cV$ are defined by $\TG$, $\TS$, and $\TV$, with no asserted one-to-one equivalence or completeness relative to classical signal representations \cite{oppenheim1999,mallat1999}.
$\cG$ requires the stated delay-embedding and reconstruction conditions; $\cS$ additionally requires the assumptions appropriate to the long-horizon statistic or stationary construction being used \cite{walters1982}; $\cV$ additionally requires local differentiability and sufficient neighborhood sampling for finite-time deformation estimation. The three domains therefore do not share a single minimal assumption set. In particular, $\cV$ requires continuous
differentiability (Condition~1's $C^1$ clause), since a local Jacobian
is well defined only where the flow is smooth. None of the three require empirical validation to be mathematically defined, conditional on their licensing conditions. The present $\cG$-anchored implementation is evaluated empirically: all $\SCSI$
features are functionals of $\Xmat$ or $x_n$, with monotonically
increasing LLE across scales (0.086--0.123). $\cS$ is not directly instantiated through the finite-state realization: $\Pmat$ and $\mu$ are never constructed; a full joint $G/S/V$ empirical instantiation remains future work. $\cV$ has application-specific evidence from AVCT
\cite{oladunni2026avct}, which reports an empirical BP calibration relationship across 263 subjects; ADT treats the affine form as a local model rather than a universal theorem.

\paragraph{Cross-Observation Transfer Limits.}
Paired ECG--PPG results illustrate why recoverability must be evaluated property by property. The Lyapunov exponent shows the strongest subject-level agreement ($\rho=0.703$), whereas $D_2$ ($\rho=0.395$) and fixed-scale $R_{\mathrm{det}}$ ($\rho=0.123$) transfer less strongly. These findings are consistent with Lemma~\ref{lem:invariance} and Corollary~\ref{cor:transfer}: observation changes can preserve some dynamical properties while altering metric-, threshold-, partition-, or noise-sensitive quantities.
The primary contribution is the ADT theory,
not a new tachypnea detector. Direct comparison with published methods
 is complicated by the absence of published
tachypnea detection AUC under a corrected evaluation
protocol: published results often use evaluation protocols that differ in subject blocking, normalization, endpoint definition, or metric, which limits direct comparison \cite{task1996heart,montano2009heart}; established RR
estimation methods use a different metric (MAE) that is not
directly comparable. Under the same corrected protocol, $\SCSI$ has a higher mean per-record AUC than the 1D CNN baseline with no trainable model parameters at inference, although the paired per-record difference is not statistically significant ($p=0.129$) (Section~\ref{sec:protocol}); a head-to-head
comparison against published RR methods on BIDMC is a natural
direction for future work.

\section{Conclusion}
\label{sec:conclusion}
Attractor Domain Theory formalizes a proposed operator-defined architecture linking Geometry, Ergodic, and Variational representations of reconstructed cardiovascular dynamics and distinguishes those representations from the descriptors measured within them. Its conditional-information result quantifies endpoint-specific descriptor contributions with an explicit residual, and its impossibility theorem identifies when a property is unrecoverable after a mathematical transformation collapses a relevant distinction. The PPG, ECG, and paired ECG--PPG evaluations show practical applicability while also demonstrating that cross-observation transfer is property-specific rather than universally invariant. ADT therefore provides a representation-level basis for asking which dynamical property is measured, what endpoint information it contributes, and what remains recoverable under a change of cardiovascular observation.

\section*{Funding Declaration}
This research received no external funding.

\section*{Competing Interests}
The authors declare that they have no competing interests.

\section*{Availability of Data and Materials}
The datasets analyzed in this study were obtained from publicly available sources. These include the Beth Israel Deaconess Medical Center (BIDMC) PPG and Respiration Dataset \cite{pimentel2017,goldberger2000physiobank}, the Brno University of Technology Smartphone PPG Database (BUT-PPG) \cite{nemcova2021}, the RWS-PPG dataset \cite{jokic2026}, the Welltory PPG dataset \cite{neshitov2021}, the CapnoBase dataset \cite{karlen2013}, and the PTB-XL electrocardiography dataset \cite{wagner2020}. The datasets are available from their respective repositories or original sources described in the cited references. No new datasets were generated as part of this study.

\end{document}